\documentclass[conference]{IEEEtran}
\IEEEoverridecommandlockouts
\usepackage{cite}
\usepackage{amsmath,amssymb,amsfonts,amsthm}
\usepackage{algorithmic}
\usepackage{graphicx}
\usepackage{textcomp}
\usepackage{subfig}

\usepackage{color}
\usepackage[linesnumbered, ruled]{algorithm2e}
\newcommand{\ie}{\textit{i.e.}}
\newcommand{\eg}{\textit{e.g.}}
\newcommand{\etal}{\textit{et al. }}
\newtheorem{definition}{Definition}

\newtheorem{lemma}{Lemma}
\newtheorem{proposition}{Proposition}

\newtheorem{theorem}{Theorem}
\newtheorem{corollary}{Corollary}
\usepackage{booktabs}
\usepackage{adjustbox}

\usepackage{multirow}
\newcommand{\hau}[1]{{\color{blue}#1}}

\def\BibTeX{{\rm B\kern-.05em{\sc i\kern-.025em b}\kern-.08em
    T\kern-.1667em\lower.7ex\hbox{E}\kern-.125emX}}
\begin{document}

\title{CARE: Compatibility-Aware  Incentive Mechanisms for Federated Learning with Budgeted Requesters}

% Federated Learning with  budgeted Requesters: Incentives for Incompatible Workers

% Federated Learning with budgeted Requesters: Compatibility-Aware Worker Incentives 
% CirFL: ceo BFL: CARE: 
% Compatibility-Aware Incentives in Federated Learning with budgeted Requesters

% Incentives for Incompatible  Workers in Federated Learning with  budgeted Requesters

% budgeted Incentive Mechanisms for Incompatible  Workers in Federated Learning

% Compatibility-Aware Federated Learning with Worker Incentives and Budget-Constrained Requesters 

% Compatibility-Aware Incentive Mechanisms for Budget-constraint Requesters  in Federated Learning

%  Incentive Mechanisms for Federated Learning with  Budgeted Requesters and Incompatible Workers

% Federated Learning with Multiple Budgeted Requesters: Compatibility-Aware Incentives for Heterogeneous Workers

% \title{BMW-FL: Budget-constrained Incentive Mechanisms for Heterogeneous Workers in Federated Learning}

% \author{}
\author{\IEEEauthorblockN{Xiang Liu$^{1,2}$, Hau Chan$^{3}$, Minming Li$^{4}$, Xianlong Zeng$^{1}$, Chenchen Fu$^{1}$, Weiwei Wu$^{1,\ast}$
\thanks{This paper has been accepted by INFOCOM-2025.}
% \thanks{$^{\ast}$Weiwei Wu is the corresponding author. This work was supported in part by the National Natural Science Foundation of China under grant 62402102, 61972086, 62272099, the Natural Science Foundation of Jiangsu Province under Grant No. BK20241275, BK20230024, BK20231543, and the Key Laboratory of Computer Network and Information Integration (Southeast University), Ministry of Education.}
}
\IEEEauthorblockA{
$^1$Southeast University, $^2$The Chinese University of Hong Kong, \\
$^3$University of Nebraska-Lincoln, $^4$City University of Hong Kong\\
% \textit{$^1$Southeast University, P.R. China} \\
% \textit{$^2$The Chinese University of Hong Kong, Hong Kong, P.R. China}\\
% \textit{$^3$University of Nebraska-Lincoln, USA}\\
% \textit{$^4$City University of Hong Kong, Hong Kong, P.R. China}\\
xiangliu@seu.edu.cn, hchan3@unl.edu, minming.li@cityu.edu.hk, \{xianlong\_zeng, chenchen\_fu, weiweiwu\}@seu.edu.cn}
% \and
% email address or ORPID}
% \and
% email address or ORPID}
% \and
% email address or ORPID}
% \and
% email address or ORPID}
% \and
% email address or ORPID}
}

\maketitle
\begin{abstract}
    Federated learning (FL) is a promising approach that allows requesters (\eg, servers) to obtain local training models from workers (\eg, clients). 
    Since workers are typically unwilling to provide training services/models freely and voluntarily, many incentive mechanisms in FL are designed to incentivize participation by offering monetary rewards from requesters.
    However, existing studies neglect two crucial aspects of real-world FL scenarios. 
    First, workers can possess inherent incompatibility characteristics (\eg, communication channels and data sources), which can lead to degradation of FL efficiency (\eg, low communication efficiency and poor model generalization).
    Second, the requesters are budgeted, which limits the amount of workers they can hire for their tasks. 
    In this paper, we investigate the scenario in FL where multiple budgeted requesters seek training services from incompatible workers with private training costs. 
    We consider two settings: the cooperative budget setting where requesters cooperate to pool their budgets to improve their overall utility and the non-cooperative budget setting where each requester optimizes their utility within their own budgets. 
    To address efficiency degradation caused by worker incompatibility,
    we develop novel compatibility-aware incentive mechanisms, CARE-CO and CARE-NO, for both settings to elicit true private costs and determine workers to hire for requesters and their rewards while satisfying requester budget constraints. 
    Our mechanisms guarantee individual rationality, truthfulness, budget feasibility, and approximation performance. 
    We conduct extensive experiments using real-world datasets to show that the proposed mechanisms significantly outperform existing baselines.
\end{abstract}

\section{Introduction}
Federated learning (FL) \cite{li2020review,kairouz2021advances} is a decentralized machine learning paradigm that enables collaborative model training across a group of workers (\eg, clients, mobile devices and data owners) without directly sharing or revealing workers' raw data openly. 
% FL has gained significant attention in recent years due to its potential to address privacy concerns and scalability issues in traditional centralized machine learning approaches.
Recently, FL has gained significant attention and has been applied to various applications in domains such as edge computing \cite{wang2019adaptive}, healthcare \cite{xu2021federated}, and finance \cite{long2020federated}.

% ,
% and many platforms (\eg, FedPro \cite{fedpro} and Sherpa \cite{sherpa}) have been designed to  facilitate FL

% , and Internet of Things (IoT) \cite{nguyen2021federated}
% \hau{I still don't get where is the real FL platform or platform? In an auction market (like eBay or alibaba is a platform), you have sellers and buyers; is there a real-world platform or system where people come in as requesters or workers? otherwise the system you talked about might not be realistic}
% \hau{what is a FL platform? companies? organizations? why do need to use the platform?} 
% liu2020fedcoin,
% \cite{jiao2020toward}\cite{zhou2021towards}
% An FL platform usually consists of task requesters (\eg, servers or model owners) who publishes the training task, and a set of workers participate in the training via their own data \cite{zhang2021survey}.

% To facilitate FL, many FL platforms (\eg, FedPro \cite{fedpro} and Sherpa \cite{sherpa}) have been designed to connect requesters (\eg, servers and model owners) and workers, who publish the training tasks and participate in the training tasks by using their local data to train local models  \cite{zhang2021survey}, respectively. 
% \hau{what do you mean by the workers distributedly? so the workers can participate in any tasks of the requesters? once each worker learns the local model then the worker sends the local model to the requester so the requester can update the global model? the sentence sounds like each worker is updating the global model? }
% In such a learning process, 
% computational resource consumption
In FL, requesters (\eg, servers, and model owners) publish their training tasks and workers  participate in the training tasks by using their local data to train local models \cite{zhang2021survey}. 
It has been observed that workers are commonly unwilling to freely contribute to training due to the costs of using their own data and computational resources \cite{tu2022incentive}. 
In addition, workers' costs are naturally private and unknown to the platform.
Therefore, previous works \cite{zhan2020learning,zhang2021incentive,zhan2021survey} have designed (truthful) incentive mechanisms in FL  to elicit workers' true private costs, select workers to hire for requesters/tasks, and determine workers' rewards/payments for training and providing local models.
% \hau{platforms?}
% allow the requester to provide monetary rewards as compensation to incentivize worker participation, 
% aiming to improve the requester's utilities (\eg, the overall reputation of selected workers \cite{zhang2021incentive} or  the global model accuracy \cite{zhan2021survey}). 
As the requesters often have limits on how much they can pay the workers, 
%considering the limited payment ability of the requester, 
recent studies have focused on designing incentive mechanisms that ensure that the total payment to workers does not exceed the requester’s budget \cite{zhang2021incentive,fan2020hybrid,sun2021pain,zhang2022online}. 
%, imposing constraints on the compensation that can be provided \cite{zhang2021incentive,fan2020hybrid,sun2021pain,zhang2022online}.

% \hau{the last two sentences; are we designing mechanisms for who? the requesters? the platform organizers? who is paying the workers? it is unclear where the budget is coming from right now}

% \hau{after describing the works above, we need to add an overview sentence about the limitations of existing studies; something like below}

% \noindent
% \textbf{Limitations of Existing Studies.} 

% \hau{either you don't say in FL or say in FL platforms}

However, existing incentive mechanism design studies in FL  do not consider two crucial aspects: (1) \textit{multiple budgeted requesters} and (2) the \textit{compatibility of workers}  that are prevalent in real-world FL.
Regarding multiple budgeted requesters, existing settings primarily concentrate on designing incentive mechanisms for a single budgeted requester \cite{roy2021distributed,xu2021bandwidth,mai2022automatic}. 
However, these approaches do not provide reasonable incentive mechanisms for multiple requesters who simultaneously seek to hire workers for their respective tasks, particularly in natural situations where each requester has limited ability to hire workers. 
% This scenario is commonly encountered in many practical settings, \eg,   crowdsensing \cite{liu2022budget} and advertising \cite{hirai2022polyhedral}.
% \hau{name some practical settings}
% In various scenarios, the intricate relationship among groups of workers (or items) often gives rise to complex situations, \ie,  the compatibility constraint among workers, which has not been addressed in the existing budget feasible mechanism design literature. 
% On one hand, many works neglect the heterogeneity between workers, \ie, they  implicitly assume  that any two workers can jointly train the global models, which is not in reality in practical applications.
% Although 
% However, there are still  challenges in designing budgeted incentive mechanism design for FL.  
% However, most these works ignore the intricate relationship among workers, \ie, the heterogeneity that some workers may not  be able to jointly perform the same training task, which is  realistic in practical applications but has not been addressed in the existing  literature.
% \hau{maybe you want to use compatibility or heterogeneity?}
Regarding the compatibility of workers,  existing incentive mechanisms in FL ignore the fact that workers can be categorized into different groups based on their inherent incompatibility characteristics. Those within the same group are incompatible in jointly performing tasks.
% workers are heterogeneous and that  there are incompatibilities among workers.
% \ie, those with similar characteristics being limited in jointly performing the same training task.
% which is common in practical scenarios but has not been addressed in the existing literature.
% The compatibility issue can lead to a decrease in the overall global model performance. 
% or providing conflicting information to a global model
% \hau{here you don't want to say request do this; first they have no power in the platform to assign workers, second if they have the power, every requester will have different conflicting constraints for the workers; you want to frame this as a platform or system that does this; i edited please check to make sure they are realistic}
% \hau{this example doesn't seem realistic; you can still order them and determine who go first and second; the example i wrote seems more realistic than the channel thing}. 

In general, compatibility issues are common in practical FL. 
For instance, in the FL  that utilizes congested wireless communication channels  (\eg, the wireless spectrum) to update global model parameters,  workers (\eg, mobile devices) can experience congestion or disconnections due to  bandwidth limitations \cite{jiao2020toward}.
To enhance the stability and efficiency of communication,  it is often advisable to restrict the number of workers using the same channel \cite{amiri2020federated}.
In such scenarios, workers are often grouped based on their specific communication channels, leading to incompatibilities between workers within the same group \cite{jiao2020toward}.
% ,zhou2013practical
% Moreover, the compatibility or compatibility constraint among workers can also arise in scenarios where there are inherent worker conflicts  due to factors such as location  \cite{zhang2019conflict} or interest preference \cite{miao2023task} conflicts. 
% Moreover, heterogeneity and  compatibility  among workers can also arise in scenarios facing generalization or robustness requirements. 
In addition, workers can also be grouped according to the similarity of their datasets and data sources (\eg, data collected from populations with varying demographics, ages or levels of education).
%in scenarios where their datasets are typically collected from domain sources with diverse preferences (\eg, data collected from populations with varying demographics, ages, or education levels). 
To improve the performance (\eg, generalization or robustness \cite{huang2024federated}) of the global model, each requester prefers a broader selection of workers from various data sources, especially within their budgets \cite{doan2001reconciling}. Consequently, workers with the same data source become incompatible when selected simultaneously.
% However, none existing  mechanisms in FL  account for budgeted requesters and  incompatible workers.
Therefore, disregarding the compatibility of workers can lead to low efficiency in FL, such as prolonged communication times during the model update process and poor generalization of the trained global model.

\noindent
\textbf{Our Goal and Contributions.} 
% \subsection{Challenges}
% Distinct from existing incentive mechanisms in FL, our problem are different in the following aspects:
Motivated by the above real-world scenarios,  we investigate the problem of designing  \textbf{C}ompatibility-\textbf{A}ware incentive mechanisms in fede\textbf{R}ated l\textbf{E}arning (CARE)  with multiple budgeted requesters and incompatible workers. 
Specifically, workers are classified into groups based on their inherent incompatibility characteristics (\eg, communication channels, and data sources).
In addition, there are compatibility constraints among workers, \ie, the number of workers assigned to the same requester/task from each group should not exceed a predefined threshold.
%We cast the CARE problem as a reverse auction (or procurement) that considers workers' private costs, compatibility constraints, and requesters' budgets.
Our goal is to design incentive mechanisms under the CARE problem to elicit workers' true private costs, select workers to hire for requesters/tasks that optimize the overall reputation (\ie, a common objective in FL \cite{kang2019incentive,zhang2021incentive}), and determine payments to the selected workers subject to the compatibility constraints and requesters' budgets.
%The objective of the designed mechanism is to select high-quality workers for requesters, specifically those with a high reputation. 
Moreover, the designed mechanisms should satisfy desirable properties, including individual rationality, truthfulness, budget feasibility, and approximation guarantees. 
We refer readers to Section \ref{sec:problem} for  justifications. 
% explanations and

% \noindent
% \textbf{Challenges.} 
Because of multiple budgeted requesters and incompatible workers, designing incentive mechanisms under the CARE problem faces three main new challenges compared to existing studies.  
\textit{1) Cost-effective worker selection}: It is efficient to prioritize workers with low-cost and high-reputation during the worker selection. However, despite their cost-effectiveness, these workers can violate compatibility constraints, making it difficult to find workers that satisfy both cost-effectiveness and compatibility.
\textit{2) Stronger strategic manipulation:} With multiple requesters, it is essential to adaptively match workers to requesters. However, this  also creates more opportunities for workers to engage in strategic manipulation, \eg, a worker may bid a false cost  to be matched with a different requester and thereby obtain higher utility. 
\textit{3) Unpredictable payments:} To satisfy requesters' budgets, we should evaluate both the reputation and the  payments  that each requester can obtain and must pay when selecting workers. However, this process is  intractable  because workers' payments remain uncertain until the final outcome of worker selection is determined, which is necessary to ensure truthfulness.
We consider the CARE problem under two realistic budget settings: \textit{(i) Cooperative budget setting:} Requesters collaborate by pooling their budgets (\eg, hospitals integrate healthcare resources such as public funds applied from the organization to train a disease recognition model  \cite{2010Using,xu2021federated}), enabling them to hire more workers and thereby enhance their overall utility (\eg, improving the model accuracy by collaboratively sharing and aggregating trained models between requesters \cite{qin2021mlmg,xu2021bandwidth}). \textit{(ii) Non-cooperative budget setting:} Each requester hires workers within their own budgets.
% For (i), we propose a mechanism that utilizes the Max-Flow method to optimize the selection of workers while considering the compatibility constraint.
% For (ii), we first divide all workers into multiple sets based on their reputations. Then, we introduce a sub-mechanism to address each worker set separately. Specially, this sub-mechanism  utilizes non-trivial virtual prices to evaluate requesters' ability to obtain reputation from workers, allowing us to ensure both budget feasibility and truthfulness.
Our main  contributions are summarized as follows:

\begin{itemize}
\item To the best of our knowledge, we are the first to design compatibility-aware incentive mechanisms in FL that capture workers' inherent incompatibilities and requesters' limited hiring abilities, thereby preventing efficiency degradation and improving budget utilization.

% while ensuring the efficient utilization of requesters' budgets to facilitate efficiency in FL.
% aiming to   restrict incompatible workers  assigned to requesters. 

% capture multiple budgeted requesters and incompatible workers in FL.
% categorized workers with compatibility constraints in FL.
% As requesters can cooperate with each other  or be selfish  \cite{tu2022incentive}, 

% we propose \textbf{C}ompatibility-aware \textbf{A}nd \textbf{B}udget-feasible incentive mechanism in \textbf{FL}, named  CARE-CO, 
\item We first propose CARE-CO mechanism  for cooperative budget setting. Particularly, CARE-CO transforms the selection of workers within the compatibility constraint into a Max-Flow problem, allowing us to explore different potential prices while simultaneously ensuring efficiency. 
We then propose CARE-NO mechanism for non-cooperative budget setting, which first divides all workers into multiple sets so that each set of workers have similar reputations. Additionally, it introduces a virtual-price based sub-mechanism, named  PEA, to address each worker set independently. Specially, PEA utilizes the concept of virtual prices to evaluate requesters’ ability to obtain reputation and determines the critical price that aligns with this ability, thereby ensuring both budget feasibility and truthfulness.
\item Our  mechanisms are proved to guarantee individual rationality, truthfulness, budget feasibility, and computational efficiency. Moreover,  our mechanisms achieve approximation guarantees  in comparison to the optimal solution that has prior knowledge of workers' private costs.

% \hau{not clear}
% \hau{this should be combined with (i) and (ii) or maybe it is ok as a separate list (?)}

% \hau{Fashion? add citations for the datasets}
\item Finally, we conduct experiments on two commonly adopted datasets in FL, \ie, Fashion MNIST (FMNIST)  and CIFAR-10.  Evaluation results show that our mechanisms improve  overall reputation of selected workers by about 824\% and the global model accuracy by about 57\% on average compared to baselines.
% The simulation results show that our mechanisms outperform existing benchmarks  in terms of the overall reputation of selected workers and the average accuracy of requesters' global models.

% \hau{in terms of what measures} .

\end{itemize}

This paper is structured as follows. Section \ref{related-work} reviews the related works.  The system model and the definition of the problem are given in Section \ref{system-model}. 
We propose    CARE-CO and    CARE-NO in Section \ref{ CARE-CO} and  \ref{ CARE-NO}, respectively.
% In Section \ref{ CARE-CO}, we propose    CARE-CO for the cooperative budget setting. 
% In Section \ref{ CARE-NO}, we introduce    CARE-NO for the non-cooperative budget setting.
Section \ref{experiment} presents the  experimental results.
% \hau{refer is wrong for the experiment section}
Finally, the conclusion is given in Section \ref{conclusion}.

\section{Related Work}\label{related-work}
Recently, a large body of literature has investigated  incentive mechanisms in FL. 
We refer the reader to the comprehensive survey \cite{zhan2021survey,tu2022incentive}. 
%The authors of \cite{zhan2021survey,tu2022incentive} present a comprehensive survey on the incentive mechanism design  in FL platform.
In the following, we focus on discussing related works on (budgeted) reverse auction-based or procurement mechanisms in FL  and  general settings.
% and budgeted reverse-auction or procurement mechanisms in
% \hau{I think you should focus on talking about settings rather than the mechanisms because the mechanisms here are for different settings; the results of the settings give you mechanisms}
% {\color{red}worker heterogeneity}

% \cite{zhan2021survey,tu2022incentive}

% \paragraph{Reverse-auction Based  Federated Learning}
% \subsection{Reverse Auction or Procurement Based Mechanisms for Federated Learning}
\textbf{Reverse Auction or Procurement Based Mechanisms for Federated Learning.}
% Reverse auction or procurement based mechanisms have been extensively used in various FL scenarios, effectively guiding the requesters in selecting high-quality workers to participate in training tasks and maximize the objective such as the social welfare , model accuracy and  the requester's utility.  
% \new{and achieving high-accuracy global models.}
% \hau{and achieving objectives?}.
Reverse auction or procurement based mechanisms have been extensively used in various FL scenarios, effectively guiding the requesters in selecting high-quality workers to participate in training tasks and maximize the objective such as the social welfare \cite{le2021incentive,jiao2020toward}, model accuracy \cite{zhang2021incentive} and  the requester's utility \cite{zeng2020fmore}.  
% For instance, 
% Le \etal \cite{le2021incentive} model the incentive mechanism between the requester and multiple workers as a reverse combinatorial auction game, 
% maximizing the social welfare.
% Jiao \etal \cite{jiao2020toward} extend to consider the mechanism under the consideration of non-IID data distribution 
% and the objective of social welfare maximization.
% Deng \etal \cite{deng2021fair} design the reverse incentive mechanism to maximize the sum of quality of aggregated model updates.  
% Zeng \etal \cite{zeng2020fmore} propose a multidimensional reverse auction mechanism 
% to optimize the total cost consumption for the requester, \ie, the cost associated with hiring workers for training the global model.
% Lu  \etal \cite{lu2023auction} propose a cluster-based client (worker) selection method and design an auction-based client selection method within each cluster to address the issue of the inefficiency associated with the random selection of clients.
While previous works have overlooked the budget constraint of the requester, Fan \etal \cite{fan2020hybrid} address this issue by considering the requester's budget and introducing a data quality-driven reverse auction.
%  thereby incentivizing both the requester and workers to participate in the FL \hau{why does it lead requester to participate in these things?}. 
This approach aims to maximize the requester's global model accuracy.
% \hau{unclear meaning}
% While these works ignore the budget of the requester, to encourage both the requester  and workers to join the FL platform, Fan \etal \cite{fan2020hybrid} further take the budget constraint of the requester into account. They introduce a data quality-driven reverse auction to maximize the valuation of the requester about the model accuracy
In a similar vein, Zhang \etal \cite{zhang2021incentive} design  a reputation calculation method  to indirectly capture the data quality of workers when designing the reverse auction-based incentive mechanism under the budget constraint.  
% \hau{what is reputation mechanism?} 
Building upon this, Zhang \etal \cite{zhang2022online} focus on the online setting of the aforementioned problems, where workers arrive in a specific order.
% Similarly, Zhang \etal \cite{zhang2021incentive} adopt the reputation mechanism to indirectly represent data quality of workers and propose  a reverse auction-based incentive mechanism under the budget constraint.
% Zhang \etal \cite{zhang2022online} pay attention to the online version of that of  \cite{zhang2021incentive} where workers arrive in an order.
% to the aforementioned studies
In addition, there are other works that examine budgeted incentive mechanisms in FL. These works explore various aspects such as differential private noises \cite{yang2023csra,ren2023differentially}, trustworthy data acquisition \cite{zheng2022fl}, and competition among federations \cite{tan2023reputation}. 
However,  all these works do not explicitly consider worker compatibility and typically assume a single requester in FL. 
% Although some papers consider the category of workers, \eg, different computational resources \cite{lim2021towards}, scheduled iterations \cite{pang2022incentive} and preferences \cite{ezzeldin2023fairfed}, they all ignore the worker compatibility and requesters' budgets in FL.
% sarikaya2019motivating

% After these, there are also some other works focus on the budgeted incentive mechanism by taking into account various aspects, \eg, differentially private  noises \cite{yang2023csra,ren2023differentially}, trustworthy data acquisition \cite{zheng2022fl} and competition among federations \cite{tan2023reputation}.
% However, these works do not consider the heterogeneity of workers and there is only one requester in the federated learning platform.
% all of the above works can be casetd to the budget feasible....
% \subsection{Budget Feasible Mechanisms}
\textbf{Budget Feasible Mechanisms.}
% \paragraph{Budget Feasible Mechanisms}
Regarding designing budgeted reverse-auction or procurement mechanisms in general settings, our work is also related to the settings of budget feasible mechanism design (see \eg, \cite{singer2010budget,chen2011approximability,gravin2020optimal,amanatidis2017budget}). 
In a typical setting of budget feasible mechanism design, a single buyer (\ie, requester) with a budget wants to procure services from sellers (\ie, workers) with private costs. 
However, all these settings assume a single buyer (requester) and disregard entirely the inherent compatibility issues among workers. 
We note that there are few works in budget feasible mechanisms consider multiple buyers \cite{wu2018budget,liu2022budget}. However, they ignore the  compatibility among workers.
%, which is not applicable to our problem.
Therefore, existing mechanisms do not  apply to our problem.

In contrast to the aforementioned previous works, our study focuses  on  incentive mechanisms for multiple budgeted requesters and incompatible workers  in FL.
%  aim to acquire  services from these workers in FL.

\section{System Model and Problem Definition}\label{system-model}
% In this section, we outline the system model and give a detailed problem definition. 
% Table \ref{tab:my_label} lists the main notations.

% \begin{table}[t]
%     \centering
%     \begin{tabular}{c|c}
%          &  \\
%          & 
%     \end{tabular}
%     \caption{Caption}
%     \label{tab:my_label}
% \end{table}

% \hau{don't you need to add some sentences or overview here? do people usually call this system model? }

\begin{figure}[t]
    \centering
    \includegraphics[width=0.35\textwidth]{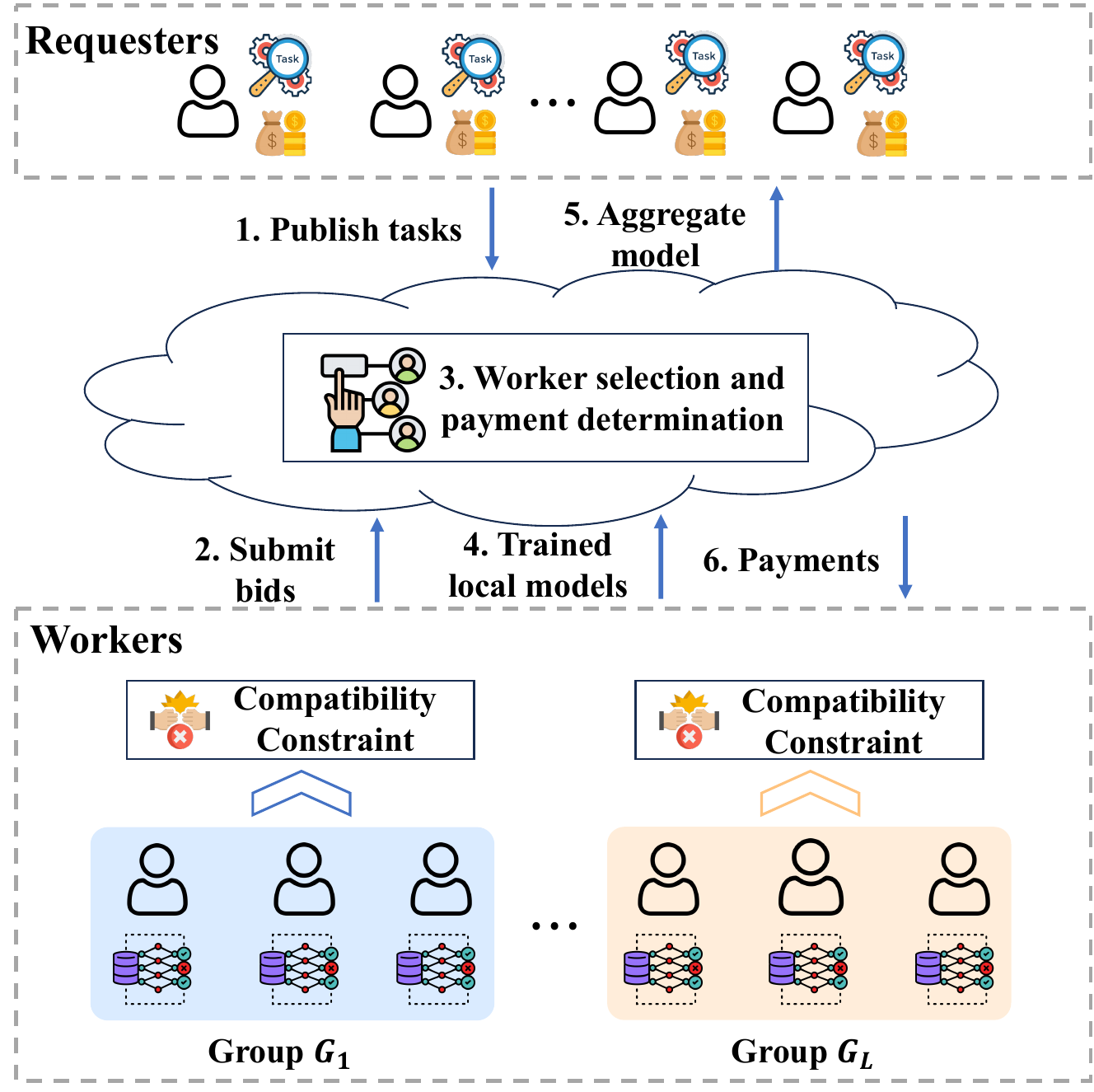}
    \caption{FL with incompatible workers and  budgeted requesters.}
    \label{fig:FL-service}
\end{figure}
% \hau{we should add more description if we have more space or time here}

\subsection{System Model}
% \hau{use set because you use groups later; here you are assuming the workers can work on all the requests}
% \hau{an FL not a FL?}
% \hau{budgets not budget in the figure; you have tasks and budget right now}

As shown in Fig. \ref{fig:FL-service}, we consider the FL system consisting of  incompatible workers and multiple budgeted requesters. 
Requesters first publish their  training tasks and workers submit their bids (\eg, each worker's bid corresponds to the cost of training local models).
Subsequently, we implement the designed incentive mechanism, which takes into account workers' bids, reputations (\eg, calculated by historical task performance \cite{zhang2021incentive}), compatibility constraints and requesters' budgets, to assign workers to requesters.
The selected workers are provided with initial global models from their corresponding requesters and  train their local models, which they then upload to requesters for aggregation. This iterative process continues until requesters' models converge. 
Finally, requesters compensate workers with monetary rewards for their services.

\subsection{Problem Definition}\label{sec:problem}

Let  $S = \{s_1, s_2, \ldots, s_n\}$  denote the set of workers who can be recruited to conduct model training tasks locally. Each worker  $s_i$  has a private raw dataset   and a cost  $c_i$  to participate in a task, \eg, the consumption of energy and computational resource. 
There are  $m$  requesters  $A = \{a_1, a_2, \ldots, a_m\}$, and each requester $a_j$ holds a training task and seeking to hire a subset of workers  $S_j\subseteq S$  to carry out the training using their  datasets\footnote{We focus on scenarios where workers have the same type of datasets, \eg, image classification or NLP datasets, while requesters seek their own models using data from these workers.  Our  mechanisms can also be readily generalized to scenarios where workers have different  dataset types by applying them separately to groups of workers with the same dataset type.}. 
% \textbf{Requester Budget Setting:}
% We naturally assume that workers possess the same type of datasets \cite{zhao2018federated}, such as handwriting or image samples.  It is worth to noting that our proposed mechanisms can be easily generalized to scenarios where workers hold different types of datasets by applying our mechanisms separately to groups of workers with the same type of dataset.
% \textbf{Requesters' Budgets and Workers Compatibility Constraints:}
Each requester  $a_j$  has a budget  $B_j$  for hiring workers, and we use  $\mathcal{B} = \{B_1, B_2, \ldots, B_m\}$  to denote requesters’ budgets.
Suppose that workers' costs are no higher than requesters' budgets \cite{wu2018budget,liu2022budget}.
% {\color{red}Depending on the behaviors of requesters \cite{tu2022incentive}, we consider two settings:
% \textit{1) Cooperative budget setting:} Requesters  collaborate by pooling their budgets (\eg, hospitals integrate healthcare resources such as public funds applied from the organization to train a  disease recognition model  \cite{2010Using,xu2021federated}), enabling them to hire more workers and thereby enhance their overall utility (\eg, improving the model accuracy by collaboratively sharing and aggregating  trained models between requesters \cite{qin2021mlmg,xu2021bandwidth}).
% Denoted by $B=\sum_{j\le m}B_j$ the total budget of all requesters.
% \textit{2) Non-cooperative budget setting:} Each requester optimizes their own utility within their individual budget.}
We consider two different budget settings: 
\textit{1) Cooperative budget setting:} Requesters are willing to collaborate by pooling their budgets, and denoted by $B=\sum_{j\le m}B_j$ the total budget of all requesters.
\textit{2) Non-cooperative budget setting:} Each requester hire workers within their individual budget.
% we consider two different settings: 
% {\color{red}\textit{(i) Cooperative budget setting:} Requesters  collaborate by pooling their budgets (\eg, hospitals integrate healthcare resources such as public funds applied from the organization to train a  disease recognition model  \cite{2010Using,xu2021federated}), enabling them to hire more workers and thereby enhance their overall utility (\eg, improving the model accuracy by collaboratively sharing and aggregating  trained models between requesters \cite{qin2021mlmg,xu2021bandwidth}).}
% \textit{(ii) Non-cooperative budget setting:} Each requester only seeks to improve their own utilities within their own budgets.
% Depending on the behaviors of requesters, we consider two budget settings: 
% \textit{1) the cooperative budget setting} where requesters are willing to collaborate by pooling their budgets, enabling them to hire  more workers and thereby enhancing their overall utilities (\eg, the overall quality of selected workers \cite{tu2022incentive} and   the model accuracy by aggregating the obtained global models among requesters \cite{xu2021bandwidth,qin2021mlmg}), and \textit{2) the non-cooperative budget setting} where each requester only wants to improve their own utilityies within their own budget \cite{deng2021fair}.
% \hau{maybe make all the textbf subsubsetion}
% \paragraph{Heterogeneous Workers and Reputation}
% \begin{sloppypar}
% \noindent

\textbf{Compatibility Constraints:}
Due to the inherent incompatibility characteristics of workers, 
% \eg, communication channel conflict \cite{jiao2020toward} and task interest conflict \cite{zhang2019conflict},
% , and the substitutable dataset from the same data source \cite{sattler2020byzantine}
workers are categorized into $L$ different groups, \ie, $\mathcal{G}=\{G_1,G_2,\cdots,G_l,\cdots,G_L\}$.
% , where $G_l$ is the set of workers with group $l$
% Let $S_{w}$ be the winning (or selected) worker set  and $S_{w,j}\subseteq S_{w}$ is the worker set allocated to requester $a_j$.
We define $\tau_{lj}$ as the compatibility level of group $G_l$ for requester $a_j$, which indicates the maximum number of workers in $G_l$ can be selected for $a_j$ (\eg, $\tau_{lj}=1$ means that only one worker in $G_l$ can be selected  for $a_j$).
Then, we  define the \textit{compatibility constraint}  such that $|S_{j}\cap G_l|\le \tau_{lj}, \forall j\le m, l\le L$.

\textbf{Incomplete Information:}
We consider the incomplete information scenario where each worker's cost is \textit{private} (known by themselves). Thus, each worker can behave strategically to misreport their private cost to improve their utility (defined below). Let $b_i$ denote the cost reported by worker $s_i$, which may not equal (or potentially much higher than) the true cost $c_i$. Denote by $\mathbf{b}=\{b_1,b_2,...,b_n\}$ and $\mathbf{b}_{-i}$  the set of   workers' bids and the set  of  workers' bids except $b_i$, respectively.

\textbf{Incentive Mechanism:}
Let $S_{w}$ be the winner (or selected) worker set, \ie, $S_w=\cup_{j\le m}S_j$.
The incentive mechanism $\mathbb{M}=(X,P)$ consists of the allocation rule $X$ which maps the bid profile $\mathbf{b}$ to  $S_w$ and the payment rule $P$ which decides the payment for each winner. Let $x_{ij}\in \{0,1\}$ indicate whether worker $s_i$ is allocated to requester $a_j$, and $x_i:=\sum_{j\le m}x_{ij}\le 1$. 
In particular, $x_i=1$ implies $s_i\in S_w$.
Let $p_{ij}$ be the payment paid to worker $s_i$ from requester $a_j$ and $p_i:=\sum_{j\le m}p_{ij}$. 
Specially, if $x_{ij}=0$, then $p_{ij}=0$.
% \hau{do you want to mention the payment to those are not selected is 0? only the seller only getting one payment right?}
Given the mechanism $\mathbb{M}$,  the utility of worker $s_i$ is the difference between the true cost and the received payment, \ie, $u_s^i(\textbf{b},\mathbb{M})=p_i-x_i\cdot c_i$.
% \paragraph{Objective}
% % \noindent

\textbf{The Objective:} 
\textit{Reputation} is commonly employed as a metric  to reflect the quality of workers in FL \cite{kang2020reliable}, \eg, their exerted efforts and the quality of their datasets.
The reputation of workers can be calculated through historical task performance \cite{kang2019incentive,zhang2021incentive}. 
Given that $v_{i}$ is the  reputation of worker $s_i$ and $\mathbf{v}=\{v_{1},v_{2},\cdots,v_{n}\}$ is the workers' reputation profile.
Thus, the objective  of the designed mechanism is to maximize the overall reputation   of selected workers $\max\sum_{i\le n,j\le m}x_{ij}v_i$, which is a standard objective in FL \cite{kang2019incentive,zhang2021incentive,zhang2022online}.
% The utility of requester $a_j$ is defined as $u_b^j(\mathcal{B})=\sum_{i\le n}x_{ij}v_i$. 
% Then, the objective  is to design mechanisms for requesters to maximize the overall utility of requesters $\max\sum_{j\le m} u_b^j(\mathcal{B})$ \cite{zhang2021incentive,zhang2022online}, \ie, overall reputation   of selected workers.
% \hau{max over what?}
% \hau{should we say the FL platform requires? can you cite something to say that is what they want?} 
Moreover, the proposed mechanism  should satisfy the following  properties:
%  \cite{deng2021fair,zhang2021incentive}: 
% maximize the sum of requesters' utilities $U$ within requesters' budgets, \ie, $\max U=\sum_{j\le m}u_b^j(\textbf{b},\mathbb{M})$.
% \begin{equation}
%     \max U=\sum_{j\le m}u_b^j(\textbf{b},\mathbb{M}).
% \end{equation}
% (**explain budget-feasible conflict-aware**)

% \hau{where do you justify reputations? i don't know if this is the correct place; also we just need to say (1) reputations are standard objectives in FL works and cite something and (2) reputations are proxies for something? i think we say (1) and (2) when you talk about the objective or utility later }
% The reputation of workers can be calculated through historical task performance \cite{zhang2021incentive, kang2020reliable}, which reflects their exerted efforts and the quality of their datasets in FL. 

\textbf{(1) Individual Rationality:} The utility of each worker $s_i$ is non-negative, \ie, $u_s^i(\mathbf{b},\mathbb{M})\ge 0$ for any $\mathbf{b}$. \textbf{(2) Truthfulness:} Bidding the true cost, each worker gets the maximum utility, \ie, $\forall b_i,u_s^i(c_i,\mathbf{b}_{-i},\mathbb{M})\ge u_s^i(b_i,\mathbf{b}_{-i},\mathbb{M})$. \textbf{(3) Budget Feasibility:} 1) For the cooperative budget setting, the total payment of requesters does not exceed the total budget $B$, \ie, $\sum_{1\le i\le n}p_i\le B$, and 2) for the non-cooperative budget setting, the total payment of requester $a_j,\forall j\le m,$ does not exceed budget $B_j$, \ie, $\sum_{1\le i\le n}p_{ij}\le B_j$. 
% {\color{red}\textbf{(4) Compatibility:} 
% The number of workers of group  $l$ allocated to the same requester cannot exceed $\tau_l$, \ie, $\forall l\le L, j\le m$, $\sum_{s_i\in G_l}x_{ij}\le \tau_l$.}
% Workers with the same group cannot be allocated to the same requester, \ie, $\forall l\le L, j\le m$, $\nexists s_{i'},s_{i''}\in G_l$ and $i'\neq i''$, such that $x_{i'j}=1,x_{i''j}=1$. 
\textbf{(4) Computational Efficiency:} The output of the mechanism can be computed in polynomial time. 
\textbf{(5) Approximation:} Let  $ALG(I)$ be the obtained reputation of  mechanism $\mathbb{M}$ with instance $I$. We compare  the output of the mechanism with the optimal  achievable overall reputation   when workers' costs are known in advance. A mechanism is  $\beta$-approximate if  {\small$\forall I,{ALG(I)} \ge  \frac{1}{\beta} {OPT(I)}$}.

\section{The Cooperative Budget Setting}\label{ CARE-CO}

% Centralized Budget/ Decentralized Budget/Preallocated/pre-distributed

% pre-determined budget, flexible budget

% fixed budget. divisible budget

% \hau{i wonder if we should use problem if we don't define problem explicitly } 

% In this section, we propose mechanism  CARE-CO.
%   for the cooperative budget setting.
% for BMW-FL problem
% propose a budgeted incentive mechanism for heterogeneous workers in federated learning ( CARE-CO). 

% \subsection{   CARE-CO}
% \subsection{   CARE-CO}
% {\color{red} Due to the challenges of BMW-FL, we first consider an easier setting that all requesters cooperatively share a global budget, and they want to hire  more workers to enhance their overall utilities \cite{xu2021bandwidth,tu2022incentive}. For this setting, we propose    CARE-CO.}
In this section, we propose    CARE-CO for the cooperative budget setting.
To address the challenge posed  by the compatibility constraint, we first disregard the requesters' budgets and assess the optimal overall reputation for a specific worker set with bids not exceeding a particular price. 
% Specifically, we design a maximum flow based method to determine the optimal allocation within the given worker set by considering the compatibility constraint. 
Subsequently, by incrementally increasing the given price, we can identify a price that maximizes the overall reputation  while ensuring budget feasibility.
The detail of   CARE-CO is  in Algorithm \ref{alg: CARE-CO}.

\begin{algorithm}[t]
    \footnotesize
\caption{\textbf{CARE-CO}($B,\textbf{b},A,S,\mathcal{G}$)} 
% \KwIn{{$B,\textbf{b},S,\mathcal{G}$. }}
\KwOut{$P,S_w$}
\label{alg: CARE-CO}
$P\leftarrow 0, S_w\leftarrow \emptyset$\;
% Sort all workers in the non-decreasing order of their  bids relative to their  reputations, \ie, $\frac{b_1}{v_1}\le\frac{b_2}{v_2}\le \cdots \frac{b_n}{v_n}$\;
Sort all workers as $\frac{b_1}{v_1}\le\frac{b_2}{v_2}\le \cdots \frac{b_n}{v_n}$\;
% $i^*\leftarrow \arg\max_{i\le n}v_i$\;
% $l\leftarrow 1$\;
% With probability $\frac{1}{3}$, \Return{$S_w=\{s_{i^*}\},p_{i^*}=B$}\;
% With probability $\frac{2}{3}$, run the following steps\;
// \textbf{Worker selection}\;
\For{$1\le i\le n$}{
    Compute $\mathcal{M}(\mathcal{S}_i)$ by Eq. (\ref{eq:maximum-value-given-worker-set}) - (\ref{cons-s-i})\;
    If $\frac{b_i}{v_i}\cdot \mathcal{M}(\mathcal{S}_i) \le B$, $i\leftarrow i+1$; otherwise, break\;
    % \eIf{$\frac{b_i}{v_i}\cdot \mathcal{M}(\mathcal{S}_i) \le B$}{
    %     $i\leftarrow i+1$\;
    % }{
    %     break\;
    % }
}
$k\leftarrow i-1$\;

Decide the allocation $X$ by the solution of $\mathcal{M}(\mathcal{S}_k)$\;

Selected workers perform training tasks for requesters\;

// \textbf{Payment scheme}\;
\For{$i\le n, j \le m$}{
    $p_{ij}=v_i\cdot x_{ij} \min\{\frac{b_{k+1}}{v_{k+1}},\frac{B}{\mathcal{M}(\mathcal{S}_k)}\}$\;
    % \eIf{$x_{ij}=1$}{
    %     $p_{ij}=v_i\min\{\frac{b_{k+1}}{v_{k+1}},\frac{B}{\mathcal{M}(\mathcal{S}_k)}\}$\;
    % }{
    %      $p_{ij}=0$\;
    % }
    % \hau{you assume that you can select at least one worker less than budget B}
} 
\end{algorithm}

As a crucial component of CARE-CO, we first introduce the method to calculate the maximum overall reputation of selected workers under the compatibility constraint, ignoring costs and budgets. We sort workers in non-decreasing order of their bids relative to their reputations, \ie, $\frac{b_1}{v_1} \le \frac{b_2}{v_2} \le \cdots \le \frac{b_n}{v_n}$. Let $\mathcal{S}_i = \{s_1, s_2, \ldots, s_i\}$ be the worker set containing the first $i$ workers. Next, we introduce the optimal overall reputation computation problem (ORP) for the  worker set $\mathcal{S}_i$.

\begin{definition}[Sub-problem ORP]
Given a worker set $\mathcal{S}_i$, then 
% ORP$(\mathcal{S}_i)$ is defined as the following integer program:
\begin{small}
    \begin{gather}
       \textbf{ORP}(\mathcal{S}_i)\textbf{:} \quad \max \sum_{t\le i,j\le m}v_tx_{tj}\label{eq:maximum-value-given-worker-set}\\
        s.t.,  \sum_{s_t\in G_{l}}x_{tj}\le \tau_{lj}, \forall j\le m,l\le L;\label{cons-hetero}\\
        \sum_{j\le m}x_{tj}\le 1,\forall t\le i; \label{cons-one-worker}\\
         x_{tj}\in \{0,1\},\forall t\le i,j\le m;\label{cons-s-i}
    \end{gather}  
\end{small}%
where (\ref{cons-hetero}) indicates compatibility constraints and (\ref{cons-one-worker}) means that each worker can only be assigned to at most  one requester\footnote{Note that all integer programs introduced in this paper  can be solved in polynomial time by converting it to the Max-Flow problem.}.
% . Note that this program can be solved in polynomial time by converting it to the Max-Flow problem (detailed in Appendix A)
% (\ref{cons-s-i}) indicates that we only consider the workers in $\mathcal{S}_i$.
% \begin{equation}\label{eq:maximum-value-given-worker-set}
%     \begin{aligned}
%     &\max \sum_{t\le i,j\le m}v_tx_{tj}\\
%     & s.t.  \sum_{s_t\in G_{l}}x_{tj}\le 1, \forall j\le m,l\le L;\text{//the compatibility constraint}\\
%     &\sum_{j\le m}x_{tj}\le 1,\forall t\le i; \text{//Each worker assigned to  one requester}\\
%     & x_{tj}\in \{0,1\},\forall t\le i,j\le m;\text{//Consider workers in $\mathcal{S}_i$}.
%     \end{aligned}
% \end{equation}
Let $\mathcal{M}(\mathcal{S}_i)$ denote the optimal  value  of  ORP$(\mathcal{S}_i)$.
\end{definition}

Then, we are ready to introduce CARE-CO. 
% \new{remove probability}
% With probability $\frac{1}{3}$, we simply choose worker $s_{i^*}$ where $i^*=\arg \max_{i\le n}v_i$, and pay $s_{i^*}$ budget $B$. With probability $\frac{2}{3}$, we do the following steps: 
We start from the first worker and find the \textit{key worker} $s_k$ such that $\frac{b_k}{v_k}\cdot \mathcal{M}(\mathcal{S}_k)\le B$ and $\frac{b_{k+1}}{v_{k+1}}\cdot \mathcal{M}(\mathcal{S}_{k+1})> B$. 
% Recall that $\mathcal{M}(\mathcal{S}_k)$ is the optimal  solution of ORP$(\mathcal{S}_k)$.
% Let $S_w(\mathcal{S}_k)$ denote the winner set in the solution of $\mathcal{M}(\mathcal{S}_k)$. 
In the optimal solution of ORP$(\mathcal{S}_k)$, if $x_i=\sum_{j\le m}x_{ij}=1$, then $s_i\in S_w$. 
% \footnote{In the event that a requester is not allocated any workers, we can rebalance the allocation by transferring one worker from requesters with multiple workers to the requester without any workers.  This rebalancing process will not have any impact on the final result.}
The payment of each worker is $p_{ij}=v_i\cdot x_{ij} \min\{\frac{b_{k+1}}{v_{k+1}},\frac{B}{\mathcal{M}(\mathcal{S}_k)}\}$.

Next, we prove the theoretical guarantees of CARE-CO.

\begin{theorem}\label{theorem:property-CARE-CO}
       CARE-CO guarantees individual rationality, truthfulness, budget feasibility and computational efficiency, and achieves a $(2+\frac{v_{max}}{v_{min}})$-approximation where $v_{max}:=\max_{i\le n}v_i$ and $v_{min}:=\min_{i\le n}v_i$.
\end{theorem}
\begin{proof}[Proof Sketch]
(1) Individual rationality: For each winner $s_i\in S_w$, we have $p_i=v_i\cdot \min\{\frac{b_{k+1}}{v_{k+1}},\frac{B}{\mathcal{M}(\mathcal{S}_k)}\}\ge v_i\cdot \frac{b_i}{v_i}$, which indicates that $s_i$' utility is $p_i-b_i\ge 0$.

(2) Budget feasibility: The total payment to the winner is $\min\{\frac{b_{k+1}}{v_{k+1}},\frac{B}{\mathcal{M}(\mathcal{S}_k)}\}\cdot \sum_{s_i\in S_w}v_i\le \frac{B}{\mathcal{M}(\mathcal{S}_k)}\cdot \mathcal{M}(\mathcal{S}_k)= B$. 

(3) Computational efficiency: The running time of    CARE-CO is dominated by  the loop in computing the optimal reputation ORP$(\mathcal{S}_i)$ in the given worker set   (line 4-7)
As ORP problem can be  converted to the Max-Flow problem,   the final total complexity of $O(MN(N+L)(ML+2N))$.

(4) Truthfulness: We leverage the famous Monotone Theorem \cite{myerson1981optimal} to prove truthfulness, which shows that truthful mechanisms satisfy monotonicity and workers are paid threshold payments. Monotonicity means that when the selected worker reports a lower cost, the worker remains selected. Threshold payments guarantee that if a worker reports a cost higher than the threshold payment, this worker will not be selected. 

% According to the definition of function $\mathcal{M}(\cdot)$, we have $\frac{b_k}{v_k}\cdot \mathcal{M}(s_{k})\le B, \frac{b_{k+1}}{v_{k+1}}\cdot \mathcal{M}(s_{k+1})>B$.
\textit{i) Monotonicity:} For any worker $s_i\in S_w$,
% $s_i$ is selected when bidding the  true cost $c_i$.  
if $s_i$ decreases their bid to ${b_i'} < {c_i}$, we prove that worker $s_i$ will still be selected. Thus,    CARE-CO satisfies monotonicity. \textit{ii) Threshold payments:}
% Recall that the payment for worker $s_i\in S_w$ is $p_i=v_i\cdot \min\{\frac{B}{\mathcal{M}(\mathcal{S}_k)},\frac{b_{k+1}}{v_{k+1}}\}$.  
According to the relationship between $\frac{B}{\mathcal{M}(\mathcal{S}_k)}$ and $\frac{b_{k+1}}{v_{k+1}}$ in the payment $p_i$, we consider two cases: $\frac{B}{\mathcal{M}(\mathcal{S}_k)}\le \frac{b_{k+1}}{v_{k+1}}$ and $\frac{b_{k+1}}{v_{k+1}}<\frac{B}{\mathcal{M}(\mathcal{S}_k)}$. Then, we prove that any winner bidding a cost higher than the threshold payment in both two cases will not obtain a higher utility. Therefore,    CARE-CO guarantees truthfulness.

(5) Approximation ratio: Let $ALG_g, OPT_g$ denote the procured reputation of    CARE-CO and the optimal solution, respectively. We divide workers into two groups: workers before $s_{k+1}$ and workers after $s_k$. For the workers before $s_{k+1}$, the optimal solution can  achieve at most $\mathcal{M}(\mathcal{S}_k)=ALG_g$ reputation with cost zero. For the workers after $s_{k}$, the optimal solution can obtain at most $\mathcal{M}(\mathcal{S}_{k+1})$ reputation under the budget constraint. Then, we prove that $\mathcal{M}({\mathcal{S}_{k+1}})-\mathcal{M}({\mathcal{S}_k})\le v_{k+1}$. Thus, we have $OPT_g\le \mathcal{M}({\mathcal{S}_k})+\mathcal{M}({\mathcal{S}_{k+1}})\le 2\mathcal{M}({\mathcal{S}_k})+v_{k+1}$, which implies  $\frac{OPT_g}{ALG_{g}}\le \frac{2\mathcal{M}({\mathcal{S}_k})+v_{k+1}}{\mathcal{M}({\mathcal{S}_k})}\le 2+\frac{v_{max}}{v_{min}}$.
\end{proof}
\section{The Non-cooperative Budget Setting}\label{ CARE-NO}

% In this section, we propose    CARE-NO for the non-cooperative budget setting.
% In the non-cooperative budget setting, ensuring budget feasibility under compatibility constraints  becomes more challenging due to the varying requesters' employability (\ie, ability to obtain reputation from workers). 
% To address this, we propose    CARE-NO. We  first introduce   PEA, a core component of  CARE-NO,  which studies the scenario where workers  have same reputations.  

In this section, we further propose CARE-NO to address the non-cooperative budget setting. 
Two key questions arise:
(1)  How can we measure requesters'  employability (\ie, the reputation they can obtain from workers)  on  varying budgets to ensure budget feasibility?
(2) Given the varying employability of requesters, how can we efficiently assign workers to requesters  under compatibility constraints and determine appropriate payments while ensuring truthfulness.

To address these two critical questions, we  first introduce a virtual-price based sub-mechanism called PEA, which serves as a core component of     CARE-NO.   PEA treats all workers as having the same reputation.
Specially, PEA introduces a non-trivial concept of  \textit{virtual price}, which helps to understand each requester's employability. 
By Utilizing the virtual price and  the requester's employability, PEA employs an integer program to assign workers under the compatibility constraint and  identifies a critical price that ensures both efficiency and truthfulness. 
Subsequently, we introduce CARE-NO, which divides all workers into multiple sets, ensuring that each set of workers has similar reputations, and applies PEA to address each worker set separately.
Detailed explanations of PEA and CARE-NO are provided in Section \ref{sec:pea} and \ref{sec:general-mechanisms}, respectively.

\subsection{Design of  PEA}\label{sec:pea}

% Sub-mechanism  PEA (Algorithm \ref{alg:pricing}) considers the scenario that  all workers  have the same reputations and aims to hire as many as possible number of workers for requesters. 

We first consider the scenario that  all workers  have the same reputations and propose   PEA (detailed in Algorithm \ref{alg:pricing})  to hire as many as possible number of workers for requesters. 
% The high-level idea is as follows. 
% We first introduce a concept of  \textit{virtual price}. The virtual price can help to determine the maximum employable number of workers for requesters  under the budget constraint, allowing us to understand each requester's employability (\ie, the number of workers they can hire).
% Utilizing the virtual price and  the requester's employability, we can obtain information regarding the maximum  number of workers that can be allocated to requesters under the compatibility constraint. 
% Then, we identify 
% a critical price that ensures both efficiency and truthfulness. 
% The detail of   PEA  is in Algorithm \ref{alg:pricing}.

\paragraph{Virtual price set}
We first introduce the concept of  \textit{virtual price} which determines the maximum employable number of workers for requesters under budget constraints.
% \textbf{Virtual price set:}
% \subsubsection{{Virtual price set.}}
We sort all workers in the non-decreasing order of their bids, \ie, ${b_1}\le{b_2}\le \cdots \le {b_n}$,
% \begin{equation}\label{eq:worker-order}
%     {b_1}\le{b_2}\le \cdots \le {b_n}
% \end{equation}
and assign a weight $w_i=2^i$ to worker $s_i$. 
Let $W=\{w_1,...,w_n\}$ denote the weight profile of the workers.
For each requester $a_j$, the employable number of workers falls within the range $[1,n]$. Denote by $\frac{B_j}{t}$ the maximum price at which requester $a_j$ can hire $t$ workers. 
Then, we can use the set of prices $\{\frac{B_j}{t}\}_{t\le n}$ to differentiate the employability of the requester $a_j$.
We define \textit{the virtual price set} $R_b=\{\frac{B_j}{t}\}_{\forall j\le m,t\le n}$ to save these prices from all requesters.
% Let $R_s=\mathbf{b}$, which aids in evaluating the set of workers whose bids do not exceed a specific price.
% Therefore, denote by $\mathcal{R}=R_s\cup R_b$ 
% \begin{equation}\label{eq:virtualset}
%     \mathcal{R}=R_s\cup R_b
% \end{equation}
% $\mathcal{R}=R_s\cup R_b$ 
% the \textit{virtual price set}.
% We remove duplicate elements  in $\mathcal{R}$ and sort elements on the non-decreasing order of their values.  
% \paragraph{Requesters' employability and the available worker set.} 
% \textbf{Requesters' employability and the available worker set:}
% \subsubsection{{Requesters' employability and the available worker set.}}
Specially, we define $\mathcal{E}(r)=\sum_{j\le m}\lfloor \frac{B_j}{r}\rfloor$ as the  \textit{requesters' employability} under price $r \in R_b$.   Let $\mathcal{S}(r)=\{s_i|b_i\le r\}$  represent the  \textit{available worker set} with bids no higher than  $r$.
% , and $M_s(r)=|\mathcal{S}(r)|$
% be the set of workers whose costs are no higher than price $r$, representing the maximum number of items under the price $r$.

\paragraph{Optimal worker selection problem (OSP)} 
Utilizing virtual prices and  the requester's employability, we next define the problem of computing the maximum  number of selected workers  under compatibility constraints at a given price. 
% \textbf{Optimal overall reputation computation problem  on a given price:}
% \subsubsection{{Optimal overall reputation computation problem  on a given price ( OSP).}}
% Given any price $r \in R_b$, we compute the maximum number of allocated workers as follows.
% For  $\forall r \in R_b$, given requesters' employability $\mathcal{E}(r)$ and the available worker  set $\mathcal{S}(r)$, we define the optimal overall reputation computation problem on a given price ( OSP)  as follows.   
\begin{definition}[Sub-problem  OSP]
    Given  $\forall r \in R_b$,  requesters' employability $\mathcal{E}(r)$ and the available worker  set $\mathcal{S}(r)$,  then
    %  OSP$(r)$ is defined as the following integer program:
\begin{small}
    \begin{gather}
        \textbf{ OSP(r):} \quad \max \sum_{s_t\in \mathcal{S}(r),j\le m}x_{tj}\label{eq:maximum-value-given-worker-set-multi-requester}\\
         s.t.  \sum_{s_t\in G_{l}\cup \mathcal{S}(r)}x_{tj}\le \tau_{lj}, \forall j\le m,l\le L;\\
        \sum_{j\le m}x_{tj}\le 1,\forall s_t\in \mathcal{S}(r);\\
         x_{tj}\in \{0,1\},\forall s_t\in \mathcal{S}(r),j\le m;\\
         \sum_{s_t\in \mathcal{S}(r)}x_{tj}\le \lfloor \frac{B_j}{r}\rfloor,\forall j\le m; \label{cons-with-ability}
\end{gather}
\end{small}%
where (\ref{cons-with-ability}) indicates that the number of workers allocated to each requester cannot exceed their employment ability at price $r$.
%  OSP$(r)$ problem can also be computed in polynomial time  by using a  Max-Flow method (deteailed in Appendix A).
Denote by $\mathcal{M}_f(r)$ the maximum  value  of  OSP$(r)$. Fig. \ref{fig:procurement-ability} illustrates the employability under price set $R_b$ and the corresponding values of  OSP$(r)$.

% similar to that in Eq. (\ref{eq:maximum-value-given-worker-set})

% \text{With ability $\mathcal\mathcal{E}(r)$}
    % \begin{equation}\label{eq:maximum-value-given-worker-set-multi-requester}
    %     \begin{aligned}
    %     \max &\sum_{s_t\in \mathcal{S}(r),j\le m}x_{tj}\\
    %      s.t. & \sum_{s_t\in G_{l}\cup \mathcal{S}(r)}x_{tj}\le 1, \forall j\le m,l\le L; \text{the compatibility constraint}\\
    %     &\sum_{j\le m}x_{tj}\le 1,\forall s_t\in \mathcal{S}(r);\text{$s_t$ Assigned to  one requester}\\
    %     & x_{tj}\in \{0,1\},\forall s_t\in \mathcal{S}(r),j\le m;\text{Workers in $\mathcal{S}(r)$}\\
    %     & \sum_{s_t\in \mathcal{S}(r)}x_{tj}\le \lfloor \frac{B_j}{r}\rfloor,\forall j\le m; \text{With ability $\mathcal\mathcal{E}(r)$}
    %     \end{aligned}
    % \end{equation}
\end{definition}

% , where the difference is that  the capacity of the edge from $a_j$ to terminal $t$ is   $\lfloor \frac{B_i}{r} \rfloor$ rather than $n$, and there are no weights on edges any more. 
% Denote the maximum procurable value  in (\ref{eq:maximum-value-given-worker-set-multi-requester}) under price $r$ as $\mathcal{M}_f(r)$. 

% Denote by $\mathcal{M}_f(r)$ the maximum  value  of  OSP$(r)$. Fig. \ref{fig:procurement-ability} illustrates the employability under price set $R_b$ and the corresponding values of  OSP$(r), r\in R_b$.
% Note that it is challenging to identify the pattern of the dynamic changes in $\mathcal{M}_f(r)$ as $r$ varies.
% For instance, when the price $r$ increases, it does not guarantee that the value of $\mathcal{M}_f(r)$ will always increase, even though the set of available workers is larger and requesters have unused capacity.
% This is due to the change in requesters' employability, \eg,  some requesters may be allocated fewer items due to their reduced employability, resulting in worker allocation.

% \paragraph{Candidate Worker Selection.}

% Specifically, if we have multiple $r$ has the same maximum value, we choose the minimum value.

% \subsubsection{Winner determination and payment scheme}
% \paragraph{{Winner determination and payment scheme.}}

\paragraph{Winner selection and payment scheme}
% \textbf{Winner determination and payment scheme:}
Given the solution of  OSP problem,  we can find the minimum price $r^*\in R_b$, namely the \textit{critical price}, such that the maximum number of allocated workers  equals requesters' employability, \ie, $r^*\in \arg\min_{r\in R_b}\{\mathcal{E}(r)=\mathcal{M}_f(r)\}$. 
We use $r_{\prec}^*$ and $r_{\succ}^*$ to denote the left and right adjacent  prices of $r^*$ in  $R_b$.
Suppose that 
% $\mathcal{M}_f(r^*)=K$,
% Note that a more aggressive critical price is $r_p\in \arg\max_{r\in R_b}\mathcal{M}_f(r)$, which however fails to ensure truthfulness, and we discuss the reason in Appendix \ref{sec:discussion}.
$s_k$ is the last worker with a cost no higher than $r^*$.

\begin{algorithm}[t]
    \footnotesize
\caption{\textbf{PEA} ($\mathcal{B},\textbf{b},A,\mathcal{G}$)} 
% \KwIn{{$\mathcal{B},\textbf{b},D_h,\mathcal{G}$. }}
\KwOut{$P,S_w$}
\label{alg:pricing}
$P\leftarrow 0, S_w\leftarrow \emptyset$\;

% $\alpha\leftarrow\min\{g_{max},m\}$\;

% \textbf{With probability $\frac{\alpha}{2\alpha+1}$:}

% \quad \quad Based on workers' indexed, we choose one worker for each requester and pay it $\min_{j\le m}{B_j}$\;  

% \textbf{With probability $\frac{\alpha+1}{2\alpha+1}$:}

% \quad \quad Run the following steps\;

% Sort all workers in the non-decreasing order of their bids, \ie, $b_1\le b_2\le \cdots b_{|D_h|}$\;
Sort all workers as $b_1\le b_2\le \cdots b_n$\;
% Assign weight $w_i=2^i$ to worker $s_i$ according to  the non-decreasing order of their bids\;
$w_i\leftarrow 2^i,\forall i\le n$\;

% // \textbf{Candidate worker selection}\;
Generate the virtual price $R_b$\;
%  according to Eq. (\ref{eq:virtualset})\;
\For{$r \in R_b$}{
    Compute $\mathcal{M}_f(r)$ by Eq. (\ref{eq:maximum-value-given-worker-set-multi-requester}) - (\ref{cons-with-ability})\;
}
% Find $K=\max_{r\in R_b} \mathcal{M}_f(r)$\;
Find $r^*\in \arg\min_{r\in\mathcal{R}_b}\{\mathcal{E}(r)=\mathcal{M}_f(r)\}$\;
% and let $K\leftarrow\mathcal{M}_f(r^*)$
% \hau{does r* always exist?}
Candidate worker set is $\mathcal{S}(r^*)$\;
// \textbf{Winner selection}\;
Compute the allocation by Eq. (\ref{eq:minimum-weight-given-worker-set-multi-requester}) - (\ref{0-1constaint}) with  worker set $\mathcal{S}(r^*)$\;
If $x_i=\sum_{j\le m}x_{ij}=1$, $s_i\in S_w$\;

Selected workers perform training tasks for requesters\;
// \textbf{Payment Scheme}\;
\For{$s_i\in S_w$}{
    \For{$b_l\ge b_i$}{
        Run PEA$(\mathcal{B},\mathbf{b}_{b_l}',S,G)$\;
        \If{$s_i$ is still selected as a winner}{
            $P_i\leftarrow P_i\cup b_{l+1}$\;
        }
    }
    $p_i=\min\{r^*,\max_{b\in P_i}\{b\}\}$\;
}
\end{algorithm}

\textbf{Winner Selection:} To ensure that any winner remains selected after bidding a lower cost (as shown in Lemma \ref{lemma:lowerweight-win}), which is crucial to guarantee the truthfulness of workers, we should choose a specific set of winners from worker set $\mathcal{S}(r^*)$.
% Given the  critical price $r^*$, we next choose winners from worker set $\mathcal{S}(r^*)$ and decide their corresponding payments.  
Particularly, we use the following integer program to select winners with the minimum sum of weights  such that the total number of selected workers is still $\mathcal{M}_f(r^*)$, \ie, 

{\begin{small}
\begin{gather}\label{eq:minimum-weight-given-worker-set-multi-requester} 
    \min \sum_{i\le k,j\le m}w_ix_{ij}\\
     s.t.   \sum_{s_i\in G_{l}}x_{ij}\le \tau_{lj}, \forall j\le m,l\le L \\
    \sum_{j\le m}x_{ij}\le 1,\forall i\le n \\
     \sum_{i\le n}x_{ij}\le \lfloor \frac{B_j}{r^*}\rfloor,\forall j\le m \label{cons-with-ability-critical}\\
    \sum_{i\le k,j\le m}x_{ij}=\mathcal{M}_f(r^*) \label{cons-choose-number}\\
    x_{ij}\in \{0,1\},\forall i\le k,j\le m \label{0-1constaint}
\end{gather}
\end{small}}%
where (\ref{cons-with-ability-critical}) indicates that the number of workers selected for each requester cannot exceed their employment ability at the price $r^*$ and (\ref{cons-choose-number}) means that we  choose $\mathcal{M}_f(r^*)$ items.
% (this integer program can still be addressed in polynomial time as detailed in Appendix \ref{sec:flow})
% This selection method helps to ensure truthfulness of workers.
% \begin{equation}\label{eq:minimum-weight-given-worker-set-multi-requester} 
%         \begin{aligned}
%         \min &\sum_{i\le k,j\le m}w_ix_{ij}\\
%          s.t.  & \sum_{s_i\in G_{l}}x_{ij}\le 1, \forall j\le m,l\le L;  \text{the compatibility constraint}\\
%         &\sum_{j\le m}x_{ij}\le 1,\forall i\le n; \text{Each worker assigned to  one requester}\\
%         & \sum_{i\le n}x_{ij}\le \lfloor \frac{B_j}{r^*}\rfloor,\forall j\le m; \text{With critical price $r^*$}\\
%         &\sum_{i\le k,j\le m}x_{ij}=\mathcal{M}_f(r^*); \text{Choose $\mathcal{M}_f(r^*)$ items}\\
%         & x_{ij}\in \{0,1\},\forall i\le k,j\le m
%         \end{aligned}
%     \end{equation}
Note that there is only one optimal solution to Eq. (\ref{eq:minimum-weight-given-worker-set-multi-requester}) since the weight of $s_i$ is $2^i$.
If $x_i=\sum_{j\le m}x_{ij}=1,\forall i\le k$, then $s_i\in S_w$.

\begin{figure}[t]
    \centering
    \includegraphics[width=0.33\textwidth]{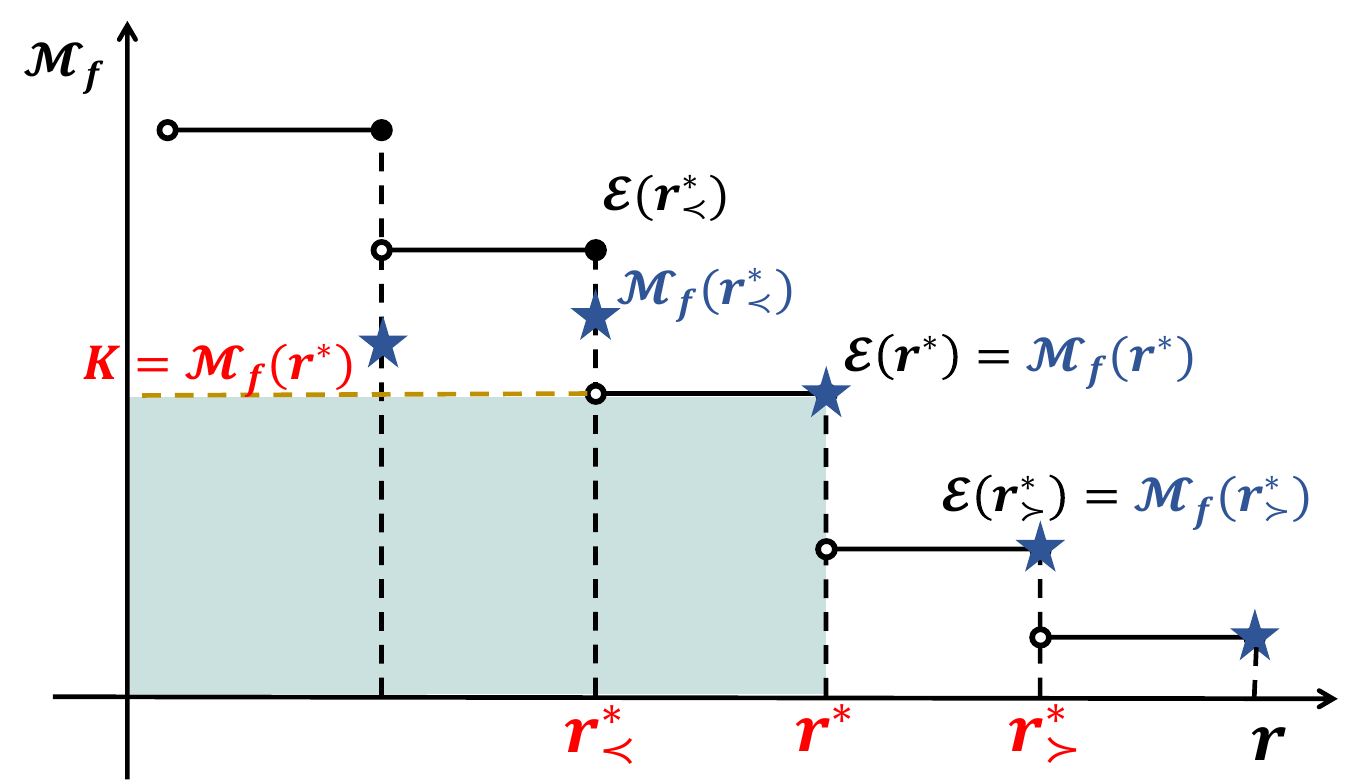}
    \caption{The maximum employable worker curves under the  price set $R_b$, where the black lines represent requesters' employability and  the blue star represents the  value of $\mathcal{M}_f(r)$.}
    \label{fig:procurement-ability}
\end{figure}

% Next, we need to decide the payment for each winner.
\textbf{Payment Scheme:}
The intuition behind the payment scheme is to determine the maximum bid that the winner $s_i$ can report while ensuring their status as a winner.
% Let $b_0=0$. 
For every winner $s_i\in S_w$,  assume that $s_i$ bids to the $l$-th position in the worker order, \ie,  bidding a higher cost $b_i'=b_l> b_i,\forall 1\le l< n$, resulting in $s_i$ becoming the new $l$-th worker in the new worker order $\mathbf{b}_{b_l}'$, \ie, $\mathbf{b}_{b_l}'=\{ b_1,\cdots,b_{i-1},b_{i+1},\cdots, b_l,b_i',b_{l+1},\cdots,b_{n_h}\}$.
% \begin{equation}
%     \mathbf{b}_{b_l}'=\{ b_1,\cdots,b_{i-1},b_{i+1},\cdots, b_l,b_i',b_{l+1},\cdots,b_n\}.
% \end{equation}
Then, we run \textbf{PEA}$(\mathcal{B},\mathbf{b}_{b_l}',S,\mathcal{G})$ with input $\mathbf{b}_{b_l}'$ for every $b_l\ge b_i$. If $s_i$ is still selected under the false cost $b_i'=b_l$, then add $b_{l+1}$ to the candidate bid set, denoted by $P_i$. Finally, the payment of $s_i\in S_w$ is $p_i=\min\{r^*,\max_{b\in P_i}\{b\}\}$. 

Next, we analyze the theoretical performance of PEA.

\begin{theorem}\label{theorem:bcs-m-property}
      PEA guarantees individual rationality, budget feasibility, and computational efficiency,  and achieves a $(2\alpha+1)$-approximation where $\alpha=\min\{m,\max_{l\le L,j\le m}\{\lceil|G_l|/\tau_{lj}\rceil\}$.
\end{theorem}
% \begin{proof}
%     \textbf{1) Individual rationality:} 
%     The payment of $s_i\in S_w$ is $p_i=\min\{r^*,\max_{b\in P_i}\{b\}\}\ge b_i$, which indicates that $s_i$'s utility is  $p_i-b_i\ge 0$.
%     \textbf{2) Budget feasibility:} 
%     The total payment of requester $a_j$ is $\sum_{i\le n}\min\{r^*,\max_{b\in P_i}\{b\}\}x_{ij}\le\sum_{i\le n}r^*x_{ij}\le r^* \lfloor \frac{B_j}{r^*}\rfloor  \le B_j$.
%     \textbf{3) Computational efficiency:} 
%     It is straightforward,  and we will not go into the details.
% \end{proof}
Before proving truthfulness of PEA, we prove the following useful lemmas.

\begin{figure}[t]
    \centering
    \includegraphics[width=0.48\textwidth]{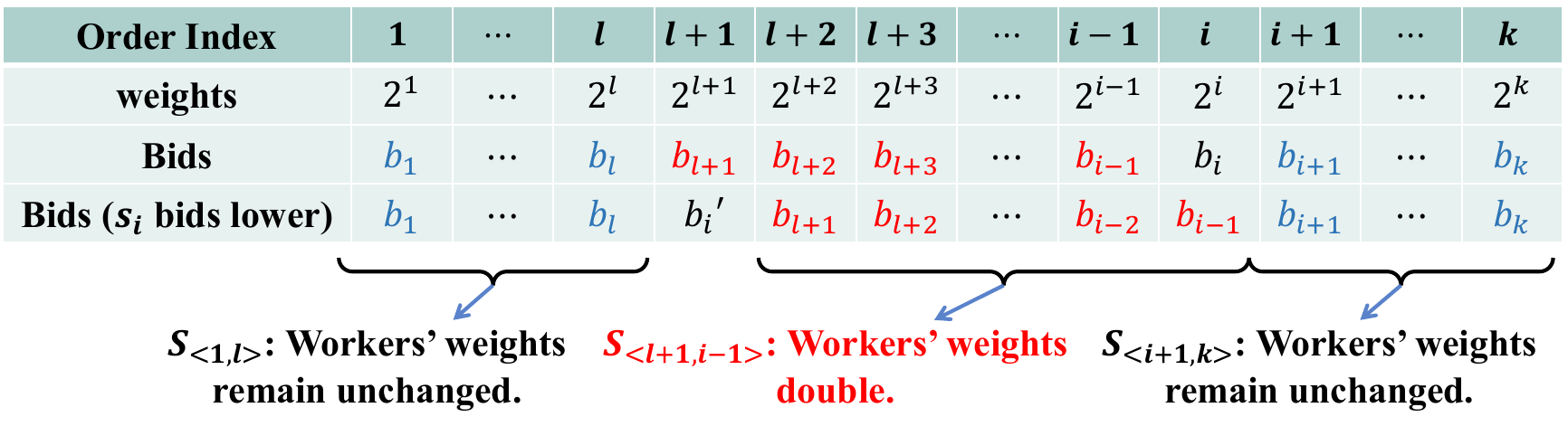}
    \caption{$k$ workers' bids and weights when $s_i$ bids a false cost.}
    \label{fig:lowerweight}
\end{figure}

\begin{lemma}\label{lemma:lowerweight-win}
    Assume that the winner $\forall s_i\in S_w$ bids a lower cost $b_i'<c_i\le r^*$. If the new critical price is still $r^*$, \ie, PEA chooses $\mathcal{E}(r^*)$ winners from worker set $\mathcal{S}(r^*)$, then $s_i$ is still a winner.
\end{lemma}
% The proof is in Appendix \ref{lemma:lowerweight-win}.

\begin{proof}
Note that $s_k$ is the last worker with the bid no higher than $r^*$.
When $s_i$ bids $b_i$, the order of the first $k$ workers' bids should be $ b_1\le b_2\le b_3\le \cdots b_k$, and the corresponding weight of the worker $s_t$ for every $t\le k$ is $w_t=2^t$.
Based on the value of $b_i'$, we consider two cases:
% Assume that there exists a virtual worker $s_0$ with bid $0$. 

\textbf{Case 1:} If $b_{i-1}< b_i'\le b_{i+1},\forall 2\le i\le k$ or $b_i'\le b_{2}, i=1$, then $s_i$ remains the $i$-th worker and the weight is still $2^i$. According to the computation of Eq. (\ref{eq:minimum-weight-given-worker-set-multi-requester}) with the input worker set $\mathcal{S}(r^*)=\{s_1,s_2,...,s_k\}$, $s_i$ is still a winner.

\textbf{Case 2:} If $b_{l}< b_i'\le b_{l+1},\forall l\le i-2, i\le k$, then $s_i$ will be the $(l+1)$-th worker with a new weight $2^{l+1}<2^i$. 
    Let $w_t',\forall t\le k,$ and $S_w'$ be the new weight of worker $s_t$ and the new winner set after $s_i$ bids a lower bid, respectively. 
    As shown in Fig. \ref{fig:lowerweight}, we can divide the workers except $s_i$ into three groups: $S_{<1,l>}=\{s_1,...,s_l\},S_{<l+1,i-1>}=\{s_{l+1},...,s_{i-1}\},S_{<i+1,k>}=\{s_{i+1},...,s_{k}\}$.
    We can find that the weights of the workers in $S_{<1,l>}$ and $S_{<i+1,k>}$ remain unchanged, whereas the weights of the workers in $S_{<l+1,i-1>}$ double, \ie, $w_t'=2w_t, \forall l+1\le t\le i-1$.
Let $\mathbb{S}_w(r^*)$ denote the set that contains all possible sets of winners computed by Eq. (\ref{eq:maximum-value-given-worker-set-multi-requester}) with price $r^*$ and $S_w\in \mathbb{S}_w(r^*)$. Thus,  $\sum_{s_t\in S_w}w_t= \min_{S\in \mathbb{S}_w(r^*)} \sum_{s_t\in S}w_t$ since $S_w$ is the final winner set.

Assume that $s_i$ is not selected as a winner after bidding at a lower cost, \ie, $S_w'\cap \{s_i\}=\emptyset$. Then, we have $\sum_{s_t\in S_w'}w_t'\ge \sum_{s_t\in S_w'}w_t$ since $w_t'\ge w_t,\forall t\le k,t\neq i$. Thus,  
\begin{equation}\label{eq:small-weight}\footnotesize
    \sum_{s_t\in S_w'}w_t'\ge \sum_{s_t\in S_w'}w_t\ge \sum_{s_t\in S_w}w_t,
\end{equation}
where the last inequality holds because $S_w$ is the winning set with the smallest sum of weights.
Let $S_{l+1}^{i-1}$ and $<i_1,i_2,...,i_{|S_{l+1}^{i-1}|}>$ denote the set  of winners selected from set $S_{<l+1,i-1>}$ when $s_i$ bids the true cost and the sequence of indexes of these workers, respectively. Then, we have  
\begin{equation*}\footnotesize
    \begin{aligned}
        \sum_{s_t\in S_{l+1}^{i-1}\cup\{s_i\}}w_t'&=2^{l+1}+2(w_{i_1}+w_{i_2}+\cdots+w_{i_{|S_{l+1}^{i-1}|}})\\
        &\le 2^{l+1}+(w_{i_2}+\cdots+w_{i_{|S_{l+1}^{i-1}|}}+2^i)\\
        &\le  w_i+ w_{i_1}+\cdots+w_{i_{|S_{l+1}^{i-1}|}}=\sum_{s_t\in S_{l+1}^{i-1}\cup\{s_i\}}w_t.
    \end{aligned}
\end{equation*}
% It is evident that $\sum_{s_t\in S_{l+1}^{i-1}}w_t\ge \sum_{s_t\in S_{l+1}^{i-1}}w_t'$ which implies 
Consequently, we can have $\sum_{s_t\in S_w}w_t\ge \sum_{s_t\in S_w}w_t'$.
% \begin{equation}\label{eq:large-weight}
%     \sum_{s_t\in S_w}w_t\ge \sum_{s_t\in S_w}w_t'.
% \end{equation}
If $\sum_{s_t\in S_w}w_t= \sum_{s_t\in S_w}w_t'$, then $S_w$ remains the winner set and $s_i$ is a winner, leading to the contradiction that $s_i$ is not selected.
If $\sum_{s_t\in S_w}w_t>\sum_{s_t\in S_w}w_t'$,
by combining Eq. (\ref{eq:small-weight}), we can conclude that $\sum_{s_t\in S_w'}w_t' > \sum_{s_t\in S_w}w_t'$. This suggests that, when worker  $s_i$ bids a lower cost, we should choose set $S_w$ in order to obtain a smaller sum of weights compared with $S_w'$, which leads to the contradiction.
Therefore, combining the above cases, this lemma holds.
% Therefore, combining the above cases, this lemma holds.
\end{proof}

\begin{lemma}\label{lemma:decreasing-function}
    $\forall r_1,r_2\in R_b,r^*\le r_1<r_2$, we have $\mathcal{M}_f(r)=\mathcal{E}(r), \forall r\in \{r_1,r_2\}$ and $\mathcal{M}_f(r_1)> \mathcal{M}_f(r_2)$.
\end{lemma}
\begin{proof}
    Based on the definition of price set $R_b$, $\mathcal{E}(r_1)> \mathcal{E}(r_2)$. As $r^*$ is the critical price, we have $\mathcal{M}_f(r^*)=\mathcal{E}(r^*)$. For every price $r\in R_b$ and $r>r^*$, we have $\mathcal{S}(r^*)\subseteq \mathcal{S}(r)$, and we can therefore assign $\mathcal{E}(r)$ workers to requesters as the mechanism can assign $\mathcal{M}_f(r^*)$ workers to requesters where $\mathcal{M}_f(r^*)=\mathcal{E}(r^*)> \mathcal{E}(r)$. Thus, we have $\mathcal{M}_f(r_1)=\mathcal{E}(r_1)>\mathcal{M}_f(r_2)=\mathcal{E}(r_2)$.
\end{proof}

\begin{lemma}\label{lemma:adding-one-worker}
    Given $\forall r\in R_b$ and the corresponding set of workers $\mathcal{S}(r)$, if we remove any worker $s_t\in \mathcal{S}(r)$, then the new value $\mathcal{M}_f(r)'$ satisfies $\mathcal{M}_f(r)'\ge \mathcal{M}_f(r)-1$.
\end{lemma}
\begin{proof}
    Let $S_w(r)$ denote the selected workers of $\mathcal{M}_f(r)$, and we have $\mathcal{S}(r)'=\mathcal{S}(r)\setminus \{s_t\}$. If $s_t\in S_w(r)$, then we can still choose $S_w(r)\setminus \{s_t\}$ that satisfies the compatibility constraints in Eq. (\ref{eq:maximum-value-given-worker-set-multi-requester}), which implies $\mathcal{M}_f(r)'= |S_w(r)\setminus \{s_t\}|=\mathcal{M}_f(r)-1$. If $s_t\notin S_w(r)$, then $\mathcal{M}_f(r)'= \mathcal{M}_f(r)$.
\end{proof}

% \begin{lemma}{\label{lemma:number}}
%     $\forall r_1,r_2\in R_b$ and $r_1<r_2$, if $r_1,r_2$ are two adjacent prices in $R_b$, then $\mathcal{E}(r_1)-\mathcal{E}(r_2)\le m$, where   $m$ is the number of requesters. 
%     \end{lemma}
% \begin{proof} 
% Assume that $\mathcal{E}(r_1)-\mathcal{E}(r_2)> m$. There must be at least one requester $a_j$ who can obtain two more items when the price changes from $r_1$ to $r_2$, \ie, $ \lfloor \frac{B_i}{r_1}  \rfloor- \lfloor \frac{B_i}{r_2}  \rfloor\geq 2$. However,  there must be a larger value $r_1+\epsilon\in (r_1,r_2)$ which satisfies that $ \lfloor \frac{B_i}{r_1}  \rfloor- \lfloor \frac{B_i}{r_1+\epsilon}  \rfloor=1$, and  $r_1+\epsilon\in R_b$ according to the definition of $R_b$. This leads to the contradiction that $r_1,r_2$ are two adjacent prices in $R_b$. Therefore, $\mathcal{E}(r_1)-\mathcal{E}(r_2)\le m$. 
% \end{proof}

% \hau{people might not know myerson}

\begin{theorem}\label{theorem-truthfulness-homogeneous}
      PEA guarantees truthfulness.
\end{theorem}
% The proof is in Appendix \ref{import-theorem}.

\begin{proof}
We still leverage the famous Monotone Theorem \cite{myerson1981optimal} so that we can prove the mechanisms to be truthful. 
Monotone Theorem shows that truthful mechanisms satisfy monotonicity and workers are paid threshold payments. 
    % We leverage famous Monotone Theorem \cite{myerson1981optimal} to  prove truthfulness, which shows that truthful mechanisms satisfy monotonicity and workers are paid threshold payments. Monotonicity means that when the selected worker reports a lower cost, the worker remains selected. Threshold payments guarantee that if a worker reports a cost higher than the threshold payment, this worker will not be selected. 
% Monotone Theorem shows that truthful mechanisms should satisfy  monotonicity and workers are paid threshold payments. 

% We prove the truthfulness by using Meyerson's theorem. 
\textit{Monotonicity:} Based on the definition of the price set $R_b$,  the value of $\mathcal{E}(r),\forall r\in R_b$, remains unchanged if any worker reports a false cost. If worker $s_i$ reports a lower cost $b_i'\le r'\in R_b$ and $b_i'<c_i$ where $r'\le r^*$, we consider two cases: \textbf{(1)}
If the new critical price remains $r^*$, $s_i$ maintains its status as the winner according to Lemma \ref{lemma:lowerweight-win}. \textbf{(2)}
If the new critical price decreases to $r'\in R_b$ where $r'< r^*$: Let $\mathcal{S}(r')'$ denote the new set of workers with costs no higher than $r'$  after $s_i$ reports a false cost. Note that $\mathcal{S}(r')'=\mathcal{S}(r')\cup \{s_i\}$. We have $\mathcal{E}(r')=\mathcal{M}_f(r')'$ as $r'$ is the new critical price, and $\mathcal{M}_f(r')<\mathcal{E}(r')$ as the critical price is $r^*$ when $s_i$ does not report a false cost. Thus, $s_i$ will also be selected as a winner. Otherwise, we must have $\mathcal{M}_f(r')=\mathcal{M}_f(r')'$ due to $\mathcal{S}(r')'=\mathcal{S}(r')\cup \{s_i\}$.
Thus, $\mathcal{M}_f(r')'<\mathcal{E}(r')$ since $\mathcal{M}_f(r')<\mathcal{E}(r')$.

\textit{Threshold payments:} Recall that the payment of  $s_i$ is  $p_i=\min\{r^*,\max_{b\in P_i}\{b\}\}$. According to the relationship between $r^*$ and $\max_{b\in P_i}\{b\}$, we consider the following two cases:

\textbf{Case 1} $p_i=r^*$: When $s_i$ bids a higher cost $r^*<b_i'$, we consider two subcase: \textbf{Subcase 1} $\mathcal{M}_f(r^*)'<\mathcal{E}(r^*)$: we have $\mathcal{S}(r^*)'=\mathcal{S}(r^*)\setminus \{s_i\}$, and we should choose a higher price $r^*{'}>r^*$ as the new critical price.
% , and $\mathcal{S}(r^*{'})'=\mathcal{S}(r^*{'})$. 
As $r^*$ is the critical price, we have $\mathcal{M}_f(r^*)>\mathcal{M}_f(r^*{'})$ according to Lemma \ref{lemma:decreasing-function}. 
If $r^*{'}<b_i'$, PEA selects the winners from the workers with a cost at most $r^*{'}$, resulting in $s_i$ not being selected as the winner.
If $r^*{'}\ge b_i'$, then it follows that $\mathcal{S}(r^*{'})'=\mathcal{S}(r^*{'})$.
As proven in  Lemma {\ref{lemma:adding-one-worker}}, $\mathcal{M}_f(r^*)'\ge \mathcal{M}_f(r^*)-1$, thus we have $\mathcal{M}_f(r^*)'\ge \mathcal{M}_f(r^*)-1\ge  \mathcal{M}_f(r^*{'})=\mathcal{M}_f(r^*{'})'$ which means that we can choose $\mathcal{M}_f(r^*{'})'$ workers from the set $\mathcal{S}(r^*)'$. 
Let $w_{\mathcal{S}(r^*)'}$ represent the maximum weight among the workers in $\mathcal{S}(r^*)'$.
Since $s_i$ bids higher than $r^*$, the weight of $s_i$ will be at least $2w_{\mathcal{S}(r^*)'}$. Consequently, $s_i$ will not be selected since we can select a minimum of $\mathcal{M}_f(r^*{'})'$ workers from the set $\mathcal{S}(r^*)'$ whose total weight is smaller than $2w_{\mathcal{S}(r^*)'}$. \textbf{Subcase 2}
$\mathcal{M}_f(r^*)'=\mathcal{E}(r^*)$: the threshold price is still $r^*$, and we will choose winners from workers with costs no higher than $r^*$, then $s_i$ will not be selected.

\textbf{Case 2} $p_i=\max_{b\in P_i}\{b\}$: If $s_i$ bids a higher cost $p_i=\max_{b\in P_i}\{b\}<b_i'\le r^*$, according to the definition of $P_i$, the worker $s_i$ will not be selected as the winner. When $s_i$ bids a higher cost $b_i'>r^*$, by \textbf{Case 1}, worker $s_i$ will also not be selected as a winner.

    Therefore, by applying Myerson's theorem,   PEA guarantees truthfulness.
    \end{proof}

\subsection{Design of  CARE-NO}\label{sec:general-mechanisms}

Given the PEA sub-mechanism, we are ready to introduce CARE-NO.
Intuitively, we divide all workers into multiple sets so that each set of workers has similar reputations. 
This way, we can treat each set of workers as having the same reputation and call   PEA to  address each worker set.
% The detail of  CARE-NO is shown in Algorithm \ref{alg:bmw-fl-non-coop}.
% The detail of  CARE-NO is as follows.
In detail,
let $\rho_i=\frac{v_i}{v_{min}}\ge 1$ represent the virtual reputation of $s_i$, and $\rho_{max}:=\max_{i\le n}\rho_i$. We divide all workers in $S$ into $\gamma=\left \lceil \log_{\varepsilon} \rho_{max} \right \rceil$ sets  $\mathcal{D}=\{D_1,...,D_\gamma\}$  by their virtual reputations and $\varepsilon>1$ is an appropriate predetermined parameter, \ie,
% (by choosing an  appropriate $\varepsilon$, we can  ensure that there is no empty set in $\mathcal{D}$)
% is a  parameter\footnote{By choosing  appropriate $\varepsilon$, we can  make sure that there is no empty set in $\mathcal{D}$.}
\begin{equation}\label{eq:group-allocation}\footnotesize
    \mathbb{D}(s_i)= \left\{\begin{matrix}
 D_h, & \text{if} \ v_i\in (\varepsilon^{h-1},\varepsilon^h], 1\le h\le \gamma\\
 D_1, & \rho_i=1,
\end{matrix}\right.
\end{equation} 
where  $\mathbb{D}(s_i)$ refers to the set to which  worker $s_i$ is selected. 
% \hau{how do you know each D is not empty?}
Specifically, workers with virtual reputation $1$ are assigned to the set $D_1$.
Then, we view that each worker in the same set owns the same reputation and call   PEA to deal with the workers in the same set.
% We then design a pricing based allocation mechanism (PEA) to deal with the workers in the same set (introduced in Section \ref{sec:pea}).
Denote by $S_w^h$ and $\mathcal{P}_h$ the  winners  and the corresponding payment returned by PEA on set $D_h$, \ie, $(\mathcal{P}_h,S_w^h)=$ \textbf{PEA}$(\mathcal{B},\textbf{b},D_h,\mathcal{G})$.
Finally, we sample one of the outputs from all sets with probability $\frac{1}{\gamma}$ as the final solution.

% Next, we analyze the theoretical properties of  CARE-NO.
\begin{theorem}\label{BCS-truthfulness}
   CARE-NO guarantees individual rationality, truthfulness, budget feasibility, and computational efficiency,  and achieves $(2\alpha+1)\varepsilon\gamma$-approximation in expectation where $\alpha=\min\{m,\max_{l\le L,j\le m}\{\lceil|G_l|/\tau_{lj}\rceil\}$ and $\gamma=\lceil\log_{\varepsilon} \rho_{max}\rceil$.
\end{theorem}

\section{Experiment}\label{experiment}
% In this section, we present experimental results to demonstrate the efficiency of the proposed mechanisms.

\subsection{Experimental Settings}

\subsubsection{Setup} To vary the quality of different workers and subsequently their reputation,
we consider the \textit{noise label datasets} \cite{li2017learning}: part of the worker data samples are incorrectly labeled, and the data accuracy rate and cost range for workers are presented in Table \ref{tab:Bid}. 
% among 10 distributed nodes, the training datasets of 5 nodes are clean, but in the training datasets of other 5 nodes, 50
% of data samples are incorrectly labeled, i.e., labels are randomly generated;
% we assign different accuracies to each worker. 
% In detail, we modify the label of a sample in order to create incorrect samples, which can result in varying accuracy for individual data.
% According to \cite{zhang2021incentive}, 
% The accuracy and cost range set for workers  is presented in Table \ref{tab:Bid}. 
% The accuracy are uniformly selected from a range between 0.1 and 1.
% The cost range for workers corresponding to different data accuracy rates is presented in Table \ref{tab:Bid}. 
Before starting the experiment, the costs are randomly generated within their corresponding cost range. 
In addition,  we first conduct $10$ training tasks to calculate the reputations of workers using the method in \cite{zhang2021incentive}.
% In addition, we first  conduct 10 training tasks on two datasets for each requester. In each task, workers are randomly assigned to requesters. This allows us to gather historical performance  for the workers. Based on this, we calculate the reputations of the workers using the reputation mechanism described in \cite{zhang2021incentive}.
We randomly assign the budget of each requester within the range [40, 80]. Furthermore, \( \tau_{lj} \) is randomly assigned within the range \( [1, |G_l|] \), and we set \( \varepsilon = 10 \) for CARE-NO. A total of 120 workers are established in FL. 
To demonstrate the influence of the number of requesters $m$, we fix the number of groups at 10 and vary the number of requesters from 2 to 12 in increments of 2. Similarly, to evaluate the impact of the number of groups $L$, we fix the number of requesters at 5 and vary the number of groups from 4 to 24 in increments of 4. 
All workers are randomly assigned to groups.
% All results are averaged over 100 rounds.
% \footnote{We includethe source codes for implementing the mechanisms  in GitHub ****.}

% \hau{how 100 workers divide evenly to them?}

% \hau{for the randomized mechanisms, are you going to say how many times you run them?}
% accuracy of requesters, we test the ****
% All results are averaged over 100 rounds.

\subsubsection{Datasets and Models} 
To validate the performance of the proposed  mechanisms, 
we consider the task on two commonly adopted datasets:  Fashion MNIST (FMNIST) and CIFAR-10. 
% The MNIST is a dataset of handwritten 0-9 digits, which contains 30,000 training samples and 30,000 test samples. 
% The Fashion MNIST is a clothing dataset containing 70,000 samples in 10 categories. 
% \hau{what about requesters? they might ask about what if you have other workers? why 100?}
For FMNIST, we adopt a three-layer neural network \cite{zhang2021incentive}, while for CIFAR-10, we use a CNN with three convolutional layers, followed by a maximum-pooling layer and two fully connected layers \cite{li2021model}.
% The input layer consists of 784 cells, the hidden layer has 50 cells, and the output layer has 10 cells.
% (each with channel size $3\times 32,32\times 64,64\times 128$ and kernel size 3$\times$3)
% \hau{why these configurations?}
% {\color{red}  For FMNIST, we adopt a three-layer neural network \cite{zhang2021incentive}, while for  CIFAR-10 is a bit more complex since it contains color information and is harder to caputure the featuers. According to \cite{li2021model},we use a similar CNN model with three convolution layer(each with channel size $3\times32,32\times 64,64\times128$ and kernel size 3$\times$3) followed by  one max pooling layer and two full-connected layer.
% The input layer consists of 784 cells, the hidden layer has 50 cells, and the output layer has 10 cells.}
For each dataset, each worker is provided with a training set of size 2000, while the requesters have test and validation datasets of size 2000 each.
Individual data is randomly drawn from the corresponding dataset.
% \hau{i don't understand how different number add up for different datasets?? you have 100 workers and they have 1000 datasets? so like examples can overlap? how do you assign examples? randomly?} 

% In addition, for each dataset, respectively, we run 10 training tasks for each requester by randomly allocating workers to requesters, to obtain workers' historical performance,  and accordingly calculate the workers' reputations by applying the reputation mechanism in  \cite{zhang2021incentive}. 
% \hau{i don't understand this last sentence? what 10 tasks? from where? how do you get the reputation? i am confused about the setup}  

% Before the start of our reputation, in order to calculate the reputation, we pre-train the requester using the method in paper \textbf{Incentive   for Horizontal FL Based on Reputation and Reverse Auction} with a slight modification to fit our setting. The differences are as follows. The requester gets the indirect reputation from neighbor requesters rather than workers. And we standardize the reputation of all requesters to be equal to 1 across the experiment. We pre-train each requester for ten rounds (select one tenths of workers without budget constraints). After that, we fix the reputation of each worker to be the average reputation of requesters.

\subsubsection{Benchmarks} 
Because no prior work has considered our challenging settings, we compare the proposed mechanisms with the following two reasonable benchmarks. 
% Because no prior works have considered our settings and the challenge of designing mechanisms for our settings, we compare the proposed mechanisms with the following two mechanism benchmarks extended from the closest related works: 
%Considering the difficulty of designing budgeted incentive mechanisms for heterogeneous workers and no prior works have provided any viable solution, 
(1) \textbf{RRAFL}: The most relevant mechanism from \cite{zhang2021incentive}, which focuses only on a single requester. We extend it to handle multiple requesters and groups of workers by assuming a virtual sum of requester budgets and randomly assigning winners to requesters without compatibility violation.
% it to address the heterogeneity by assuming a virtual requester with the sum of requesters’ budgets and allocating winners randomly to the requester without the heterogeneity violation. \hau{i don't understand this sentence? what about multiple requesters? You need to say we extend it to handle multiple requesters and hetero workers by XXXX XXXX, respectively.}
(2)  \textbf{RanPri}: A simple pricing mechanism that sets a random price within the cost range for each worker. If the price is at least equal to the worker's cost, the worker is selected as a winner and assigned to a random requester with sufficient remaining budget, ensuring no compatibility violations.

\subsubsection{Metrics} 

We evaluate these mechanisms using the following two metrics: \textit{(1) Overall Reputation}:  $\sum_{i\le n}\sum_{j\le m}x_{ij}v_i$, which is the objective of our proposed mechanisms; \textit{(2) Average Global Accuracy}: the average global model accuracy of requesters, \ie, $\frac{1}{m}\sum_{j\le m}q_i$ 
where $q_i$ is the global model accuracy of requester $a_j$.

\begin{table}[t]
    \caption{Bid Ranges w.r.t.   Data Accuracy Rates.}
    \centering
    \begin{tabular}{@{}c|ccc@{}}
    \toprule
    \textbf{Data accuracy rate} & [0.4, 0.6) & [0.6, 0.8) & [0.8, 1.0{]} \\ \midrule
    \textbf{Bid range}          & {[}2, 4{]}   & {[}3, 5{]}   & {[}4, 6{]}     \\ \bottomrule
    \end{tabular}
    \label{tab:Bid}
    \end{table}

\begin{figure*}[t]
    \centering
    \subfloat[\#requesters for FMNIST]{
        \label{fig:requesters-FMNIST}
        \includegraphics[width=0.235\textwidth]{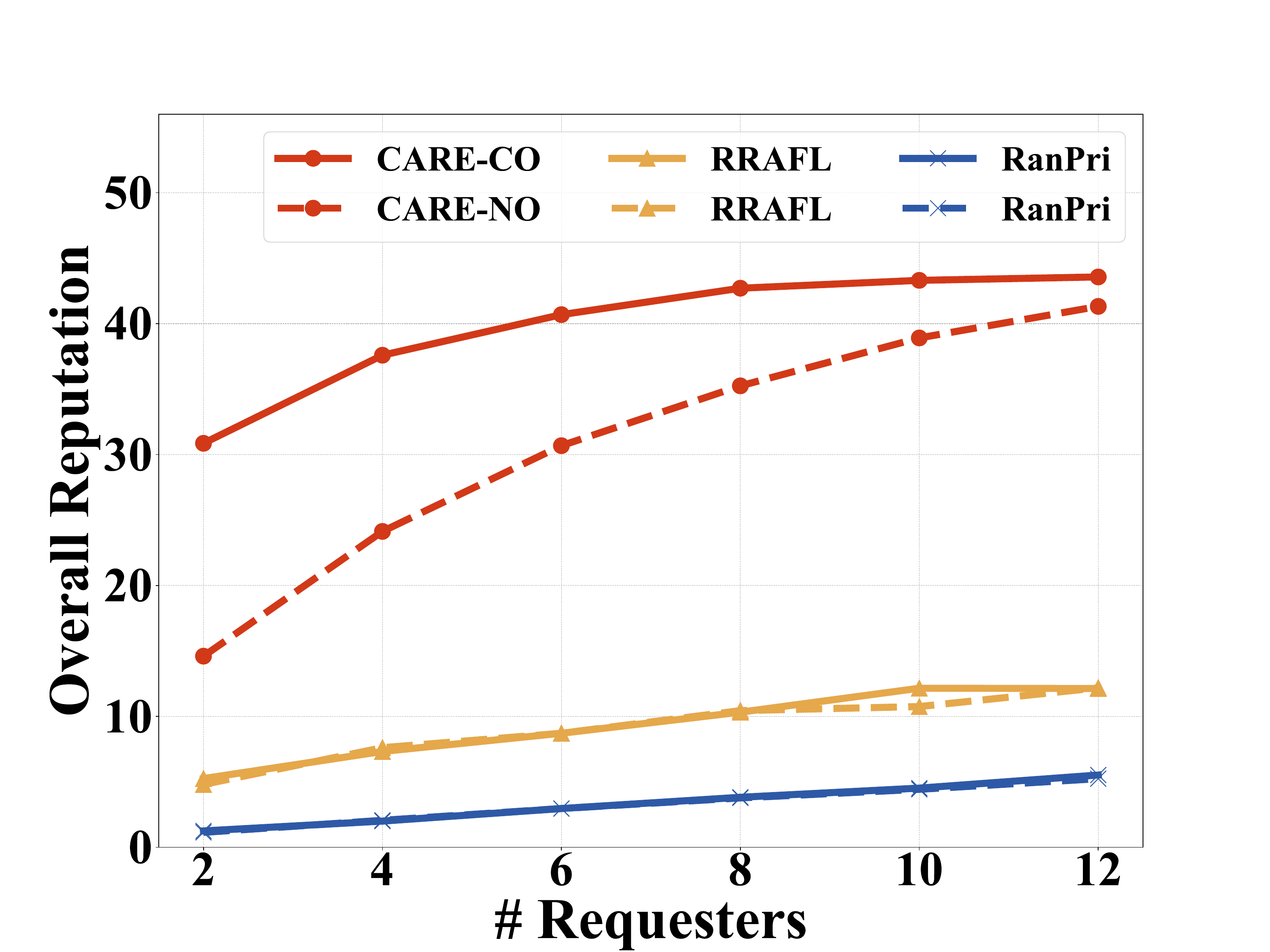}
    }
    \subfloat[\#groups for FMNIST]{
        \label{fig:groups-FMNIST}
        \includegraphics[width=0.235\textwidth]{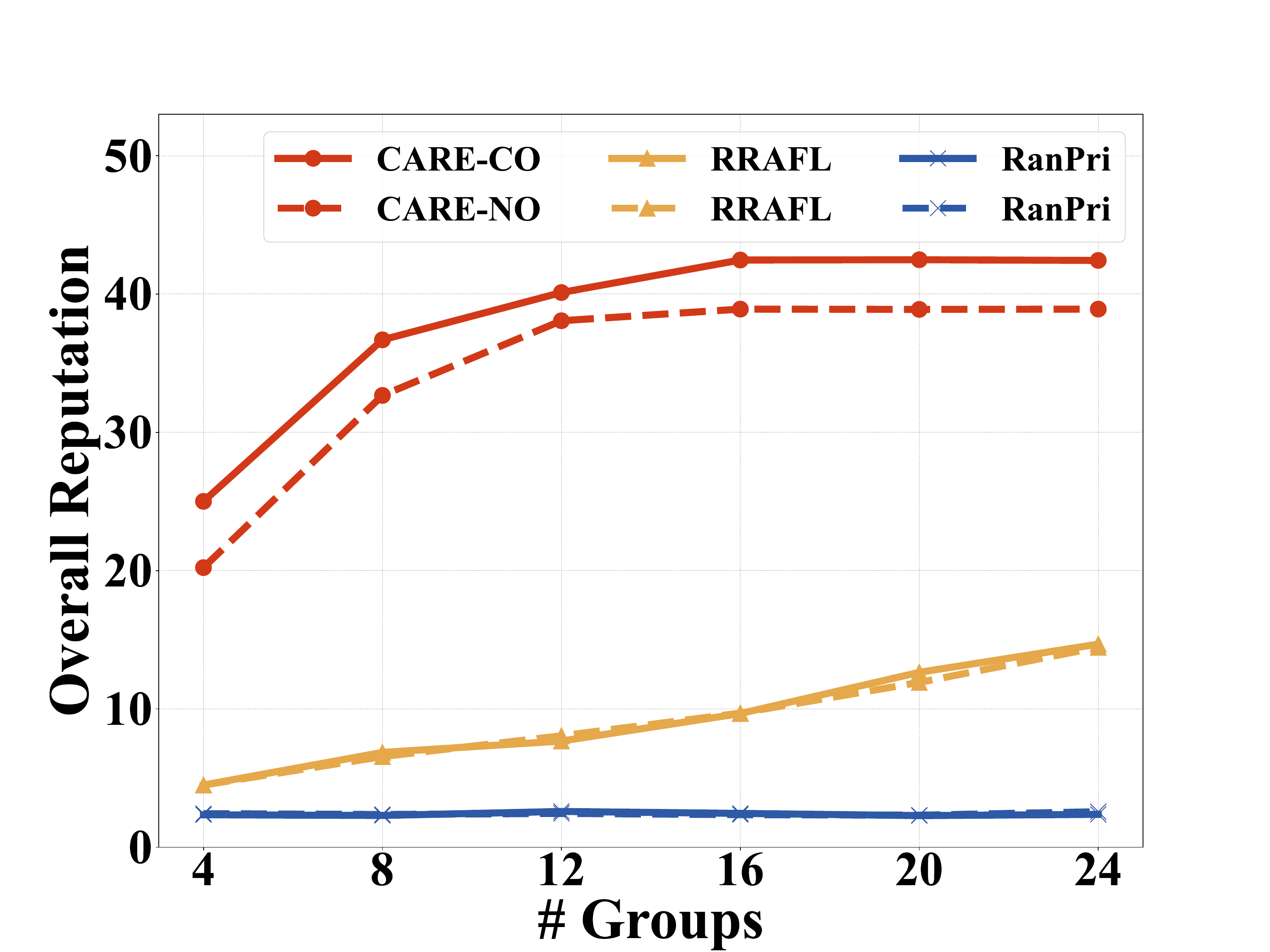}
    }
    \subfloat[\#requesters for CIFAR-10]{
        \label{fig:requesters-CIFAR-10}
        \includegraphics[width=0.235\textwidth]{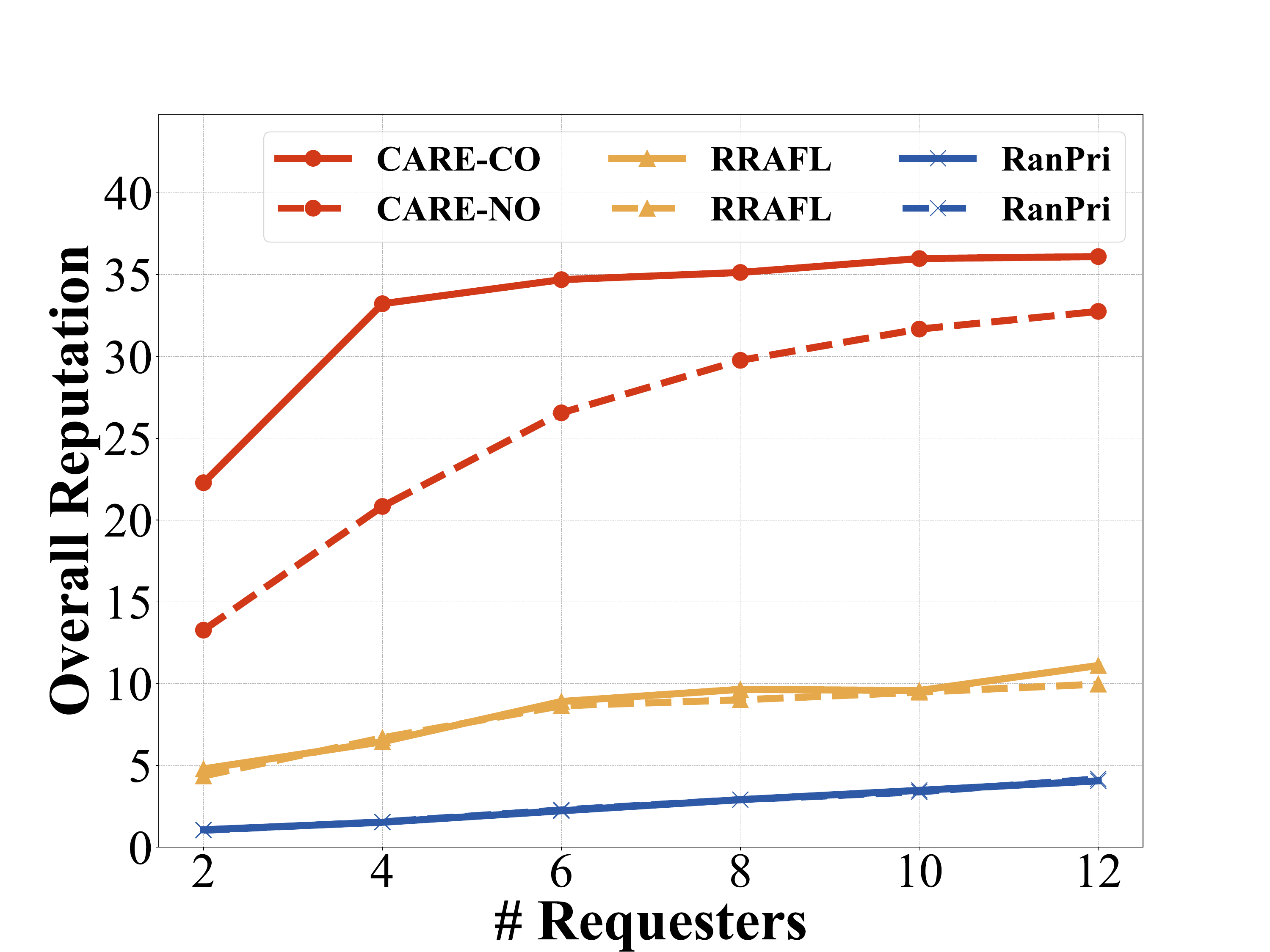}
    }
    \subfloat[\#groups for CIFAR-10]{
        \label{fig:groups-CIFAR-10}
        \includegraphics[width=0.235\textwidth]{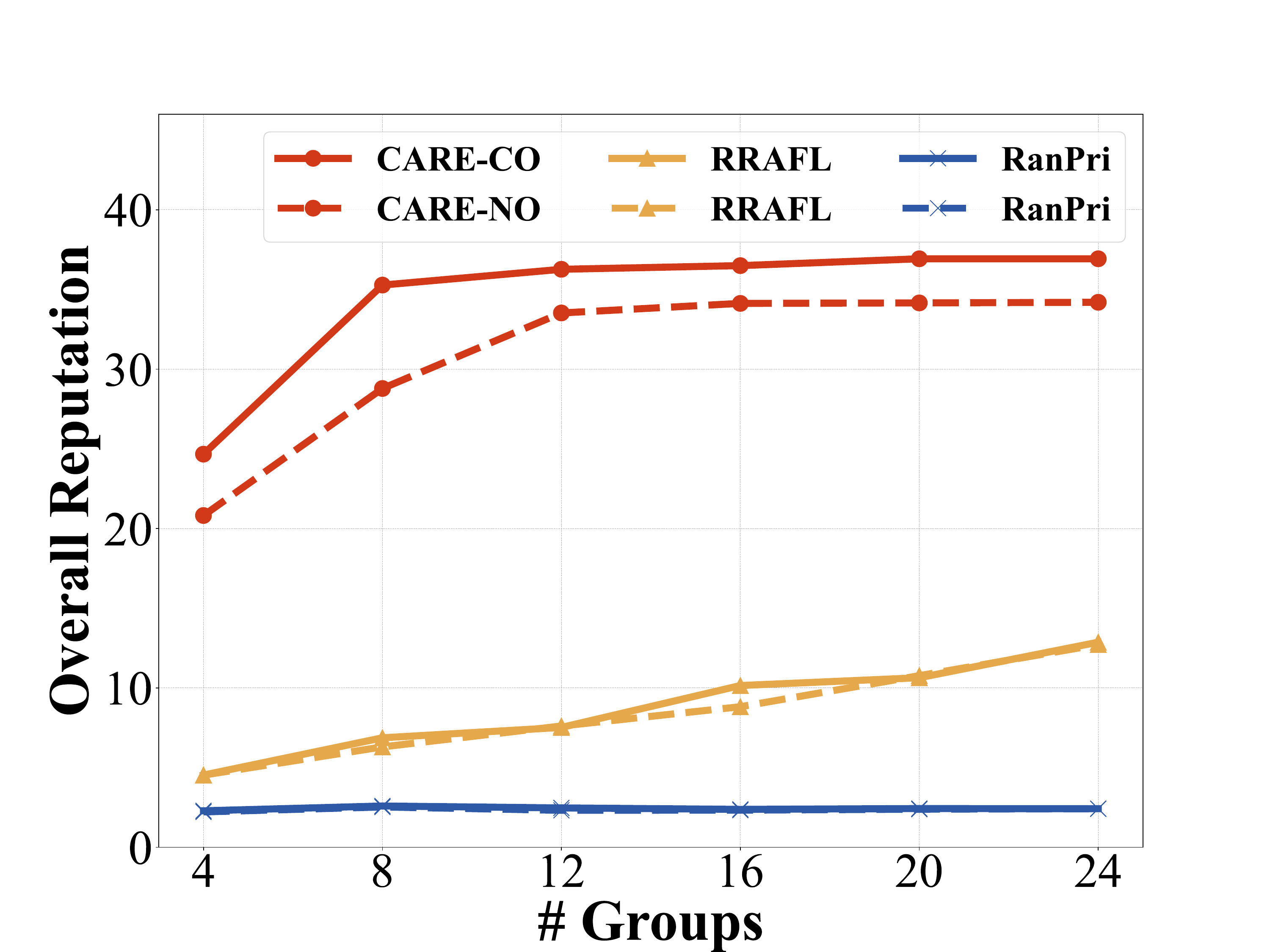}
    }
    \caption{The  overall reputation of selected workers, with solid lines representing the reputation in the cooperative budget setting and dashed lines representing the reputation in the non-cooperative budget setting.}
    \label{fig:overallreputationdata}
\end{figure*}

\subsection{Experimental Results}
% We introduce the experimental results on the overall reputation and the average accuracy in Section \ref{exp:reputation} and \ref{exp:accuracy}, respectively.

% \hau{need to say something here about different subsubsecions}

% \hau{for the table you need to say the mechanism name not just BFL}

% \hau{fix table typos on data and check other typos}

\subsubsection{Overall Reputation of Selected Workers} \label{exp:reputation}

Fig. \ref{fig:overallreputationdata} presents the overall reputation of the proposed mechanisms for different numbers of requesters and groups. 
Specifically, Figs. \ref{fig:requesters-FMNIST}, \ref{fig:groups-FMNIST} and Figs. \ref{fig:requesters-CIFAR-10}, \ref{fig:groups-CIFAR-10} show the results on FMNIST and CIFAR-10, respectively.
\textbf{Firstly}, we observe that our proposed mechanisms, CARE-CO and CARE-NO, consistently outperform RRAFL and RanPri. 
In particular, the overall reputation of CARE-CO is significantly higher, with average improvements of $3$ and $15$ times compared to RRAFL and RanPri, respectively. Similarly, the overall reputation of CARE-NO has improved by factors of $2$ and $12$, respectively, compared to RRAFL and RanPri.
% CARE-CO exhibits a significantly higher overall reputation, with improvements of 3 and 15 times  compared to RRAFL and RanPri on average, respectively. 
% Similarly, CARE-NO demonstrates overall reputation improvements of 2 and 12 times compared to RRAFL and RanPri, respectively. 
The superior performance of our proposed mechanisms can be attributed to their consideration of worker cost-efficiency (\ie, bids relative to reputations) and their ability to efficiently allocate the budget to select more efficient workers while accommodating worker compatibility. In contrast, RRAFL and RanPri struggle to allocate workers effectively in the presence of multiple budgets and worker compatibility.
\textbf{Secondly}, we consistently find that the overall reputation of CARE-CO exceeds that of CARE-NO. This aligns with the intuitive understanding that sharing the budget can facilitate the hiring of more high-quality workers, thereby enhancing the overall reputation attained by requesters.
\textbf{Lastly}, our proposed mechanisms demonstrate a stable increase in overall reputation as the number of groups increases. This can be attributed to the fact that a higher number of groups indicates a lower level of incompatibilities among workers, which in turn contributes to achieving a higher overall reputation.

\subsubsection{Average Global Accuracy of Requesters}
\label{exp:accuracy}

% \hau{see respectively comments }
To validate the effectiveness of our proposed mechanisms, which not only exhibit high overall reputation efficiency but also improve model accuracy, we present the performance of average global accuracy for requesters as follows.
Table \ref{table:accuracy} shows the impact of the number of requesters and groups on the average global accuracy for both FMNIST and CIFAR-10.
\textbf{Firstly},  it is evident that the proposed mechanisms, CARE-CO and CARE-NO, consistently outperform RRAFL and RanPri. Specifically, the average accuracy of CARE-CO is 32.21\% and 84.49\% higher than that of RRAFL and RanPri on average, respectively. Similarly, the average accuracy of CARE-NO is 28.70\% and 83.89\% higher than that of RRAFL and RanPri on average, respectively.
\textbf{Secondly},  it is worth noting that the improvement in accuracy achieved by our proposed mechanisms is not as significant as the improvement in overall reputation. This is because achieving higher accuracy requires the participation of more high-quality workers. Nevertheless, both CARE-CO and CARE-NO are capable of achieving significantly better average accuracy due to their considerably higher reputation.
\textbf{Lastly}, we observe a slight decrease in average accuracy with an increasing number of requesters. This can be attributed to the fact that as the number of requesters increases, the number of workers allocated to each requester may slightly decrease, resulting in a decrease in average accuracy. Moreover, the overall average accuracy of our proposed mechanisms increases as the number of groups increases. The reason is that our proposed mechanisms can attain higher reputation in the presence of a lower level of incompatibilities among workers (\ie, a larger number of groups), which in turn contributes to achieving higher model accuracy for the requesters.

\begin{table*}[t]
    \caption{The Performance of the Average Global Accuracy}
    \label{table:accuracy}
    \begin{adjustbox}{width=\textwidth}
    \begin{tabular}{@{}l|c|ccc|ccc|c|ccc|ccc@{}}
    \toprule
    \multicolumn{1}{c|}{}                                                &         & \multicolumn{3}{c|}{Cooperative} & \multicolumn{3}{c|}{Non-cooperative} &         & \multicolumn{3}{c|}{Cooperative} & \multicolumn{3}{c}{Non-cooperative} \\ \midrule
    \multicolumn{1}{c|}{}                                                & \# Req. & CARE-CO          & RPAFL & RanPri & CARE-NO           & RPAFL   & RanPri  & \# Group & CARE-CO          & RPAFL & RanPri & CARE-NO           & RPAFL  & RanPri  \\ \midrule
    \multirow{6}{*}{\begin{tabular}[c]{@{}l@{}}FMN-\\ IST\end{tabular}}  & 2       & \textbf{0.797}  & 0.796 & 0.588  & \textbf{0.775}   & 0.759   & 0.653   & 4       & \textbf{0.784}  & 0.571 & 0.573  & \textbf{0.779}   & 0.664  & 0.570   \\
                                                                         & 4       & \textbf{0.795}  & 0.694 & 0.542  & \textbf{0.787}   & 0.742   & 0.551   & 8       & \textbf{0.781}  & 0.612 & 0.513  & \textbf{0.782}   & 0.599  & 0.494   \\
                                                                         & 6       & \textbf{0.786}  & 0.654 & 0.519  & \textbf{0.779}   & 0.653   & 0.531   & 12      & \textbf{0.787}  & 0.648 & 0.509  & \textbf{0.785}   & 0.621  & 0.507   \\
                                                                         & 8       & \textbf{0.790}  & 0.504 & 0.516  & \textbf{0.782}   & 0.657   & 0.488   & 16      & \textbf{0.789}  & 0.620  & 0.547  & \textbf{0.782}   & 0.627  & 0.493   \\
                                                                         & 10      & \textbf{0.782}  & 0.691 & 0.559  & \textbf{0.781}   & 0.621   & 0.501   & 20      & \textbf{0.798}  & 0.683 & 0.564  & \textbf{0.785}   & 0.638  & 0.553   \\
                                                                         & 12      & \textbf{0.789}  & 0.560 & 0.507  & \textbf{0.781}   & 0.568   & 0.507   & 24      & \textbf{0.797}  & 0.741 & 0.537  & \textbf{0.793}   & 0.670  & 0.535   \\ \midrule
    \multirow{6}{*}{\begin{tabular}[c]{@{}l@{}}CIFA\\ -R10\end{tabular}} & 2       & \textbf{0.546}  & 0.422 & 0.296  & \textbf{0.537}   & 0.477   & 0.243   & 4       & \textbf{0.549}  & 0.340 & 0.246  & \textbf{0.533}   & 0.340  & 0.249   \\
                                                                         & 4       & \textbf{0.553}  & 0.431 & 0.211  & \textbf{0.518}   & 0.427   & 0.235   & 8       & \textbf{0.548}  & 0.355 & 0.263  & \textbf{0.536}   & 0.392  & 0.258   \\
                                                                         & 6       & \textbf{0.543}  & 0.398 & 0.227  & \textbf{0.508}   & 0.375   & 0.219   & 12      & \textbf{0.559}  & 0.391 & 0.250  & \textbf{0.555}   & 0.357  & 0.251   \\
                                                                         & 8       & \textbf{0.539}  & 0.375 & 0.354  & \textbf{0.511}   & 0.368   & 0.249   & 16      & \textbf{0.555}  & 0.378 & 0.221  & \textbf{0.558}   & 0.444  & 0.240   \\
                                                                         & 10      & \textbf{0.533}  & 0.373 & 0.233  & \textbf{0.538}   & 0.364   & 0.237   & 20      & \textbf{0.551}  & 0.460 & 0.239  & \textbf{0.553}   & 0.421  & 0.276   \\
                                                                         & 12      & \textbf{0.523}  & 0.316 & 0.232  & \textbf{0.514}   & 0.341   & 0.238   & 24      & \textbf{0.556}  & 0.473 & 0.233  & \textbf{0.559}   & 0.444  & 0.234   \\ \bottomrule
    \end{tabular}
\end{adjustbox}
    \end{table*}

\begin{figure}
    \centering
    \includegraphics[width=0.48\textwidth]{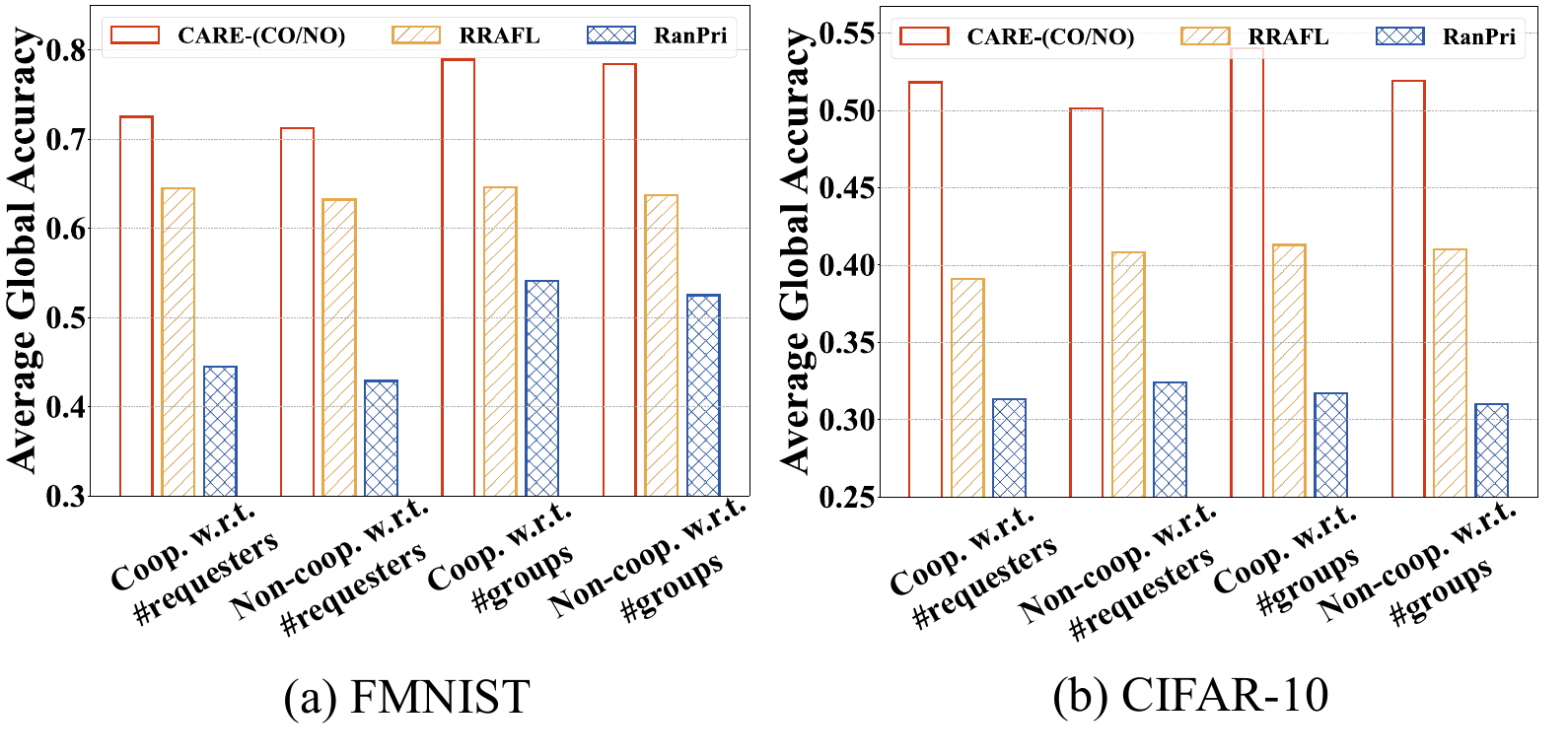}
    \caption{The average global accuracy of the proposed mechanisms under non-IID label distribution datasets.}
    \label{fig:non-iid-accuracy}
\end{figure}

\subsubsection{Accuracy on Non-IID Label Distribution} 
In the previous experiments,  we have  considered the performance of the proposed mechanisms in noisy label datasets where noise was added to create incorrect samples. In order to further validate the robustness and efficiency of our proposed mechanisms, we also compare the average global model accuracy with baselines on \textit{non-IID label distribution} datasets \cite{wang2020optimizing}, \ie, the sample labels in the dataset of the workers are non-uniformly distributed.
Fig. \ref{fig:non-iid-accuracy} illustrates the average accuracies of the proposed mechanisms in the FMNIST and CIFAR-10 datasets. 
We observe that the average global model accuracies of the proposed mechanisms decrease by only 4.25\% and 3.75\% on average compared to the accuracy obtained in Section \ref{exp:accuracy} for the FMNIST and CIFAR-10 datasets, respectively. This further validates the robustness of our proposed mechanisms on non-IID label distribution datasets.
Furthermore, our proposed mechanisms continue to significantly outperform the baseline mechanisms in non-IID label distribution datasets. Specifically, CARE-CO shows an improvement of 24.01\% and 61.13\% compared to RRAFL and RanPri, respectively, in terms of average accuracy. Similarly, CARE-NO demonstrates an improvement of 21.28\% and 59.33\% compared to RRAFL and RanPri, respectively.

% This further validates the proposed mechanisms significantly outperform the baseline mechanisms.

% \hau{please make sure you use the right wording about heterogeneity; double check}

% \usepackage{booktabs}
% \usepackage{multirow}

\section{Conclusion}\label{conclusion}
In this paper, we consider compatibility-aware incentive mechanisms for incompatible workers in FL, where multiple requesters with budgets want to procure training services from groups of workers. %, \ie,  a single requester will not be allocated any two workers from the same conflict group.
%the budget constraint, simultaneously. 
% \hau{sometime you use max flow something else in the paper; please make sure they are consistent}
For the cooperative budget setting, we propose CARE-CO, which leverages the Max-Flow solution to find feasible allocations under the compatibility constraint. 
Then, we propose CARE-NO for non-cooperative
budget setting, which divides all workers into multiple sets and introduces a sub-mechanism PEA to address each worker set separately. 
In particular, PEA uses virtual prices to evaluate requesters’ ability to obtain reputation and determines the critical price that aligns with their ability to ensure both budget feasibility and truthfulness.  
The proposed mechanisms can ensure individual rationality, budget feasibility, truthfulness, and approximation guarantee.
Experimental results in real-world datasets, FMNIST and CIFAR-10, validate that our proposed mechanisms significantly outperform baseline mechanisms in terms of the overall reputation of selected workers and the average global accuracy.

\bibliographystyle{IEEEtran}
\bibliography{reference}

\end{document}